\documentclass[11pt,twoside]{article}
\usepackage[margin=1in]{geometry}
\usepackage{amsmath,amssymb, amsthm,epsfig,bm}
\usepackage{mathtools}
\usepackage[shortlabels]{enumitem}
\usepackage{caption}
\usepackage{subcaption}
\usepackage{natbib}
\usepackage[colorlinks=true,citecolor=blue,linkcolor=blue]{hyperref}
\usepackage{algorithm, algpseudocode}
\usepackage[table,xcdraw]{xcolor}
\usepackage{booktabs}
\usepackage{dsfont}
\usepackage{footnote}
\usepackage{centernot}
\usepackage{graphicx}
\usepackage{color}
\usepackage{tikz}
\usetikzlibrary{shapes,arrows.meta}
\usepackage[greek,english]{babel}

\setlist{noitemsep, topsep=0pt}
\bibpunct{(}{)}{;}{a}{}{,} 

\numberwithin{equation}{section}
\theoremstyle{plain}
\newtheorem{theorem}{Theorem}
\newtheorem{lemma}{Lemma}

\theoremstyle{definition}
\newtheorem{definition}{Definition}

\newtheorem{remark}{Remark}

\def\H{{H}}  % nothing
\def\I{{I}}  % nothing
\newcommand{\bfA}{\mathbf{A}}
\newcommand{\bfB}{\mathbf{B}}
\newcommand{\bfC}{\mathbf{C}}
\newcommand{\bfE}{\mathbf{E}}
\newcommand{\bfN}{\mathbf{N}}
\newcommand{\bfS}{\mathbf{S}}
\newcommand{\bfU}{\mathbf{U}}
\newcommand{\bfV}{\mathbf{V}}
\newcommand{\bfX}{\mathbf{X}}
\newcommand{\bfc}{\mathbf{c}}
\newcommand{\bfk}{\mathbf{k}}
\newcommand{\bfx}{\mathbf{x}}
\newcommand{\calD}{\mathcal{D}}
\newcommand{\calG}{\mathcal{G}}
\newcommand{\calK}{\mathcal{K}}

\newcommand{\rmPr}{\mathrm{Pr}}
\newcommand{\hcalG}{\hat{\mathcal{G}}}
\newcommand{\tdcalG}{\tilde{\mathcal{G}}}

\newcommand{\pa}[1]{{\bm \pi}_{#1}^\mathcal{G}}
\newcommand{\Pa}[1]{{\bm \Pi}_{#1}^\mathcal{G}}

\newcommand{\sep}{\!\perp\!\!\!\!\perp\!}
\newcommand{\dsep}[3]{(#1 \sep #2 \!\mid\! #3)_{\mathcal{\calG}}}
\newcommand{\notdsep}[3]{(#1 \not\sep #2 \!\mid\! #3)_{\mathcal{G}}}
\newcommand{\indep}[3]{(#1 \sep #2 \!\mid\! #3)_P}
\newcommand{\notindep}[3]{(#1 \not\sep #2 \!\mid\! #3)_P}
\newcommand{\indepinfo}{{\{ {\perp\!\!\!\!\perp}_P \}}}

\newcommand{\ind}[1]{\mathds{1} \left\{ #1 \right\}}
\providecommand{\abs}[1]{\lvert#1\rvert}

\DeclareMathOperator*{\argmin}{argmin}
\DeclareMathOperator*{\argmax}{argmax}

\DeclareFontFamily{U} {MnSymbolA}{}
\DeclareFontShape{U}{MnSymbolA}{m}{n}{
  <-6> MnSymbolA5
  <6-7> MnSymbolA6
  <7-8> MnSymbolA7
  <8-9> MnSymbolA8
  <9-10> MnSymbolA9
  <10-12> MnSymbolA10
  <12-> MnSymbolA12}{}
\DeclareFontShape{U}{MnSymbolA}{b}{n}{
  <-6> MnSymbolA-Bold5
  <6-7> MnSymbolA-Bold6
  <7-8> MnSymbolA-Bold7
  <8-9> MnSymbolA-Bold8
  <9-10> MnSymbolA-Bold9
  <10-12> MnSymbolA-Bold10
  <12-> MnSymbolA-Bold12}{}
\DeclareSymbolFont{MnSyA} {U} {MnSymbolA}{m}{n}
\DeclareMathSymbol{\leftrightline}{\mathrel}{MnSyA}{208}
\newcommand{\adjacent}{\leftrightline}

\begin{document}

\title{Partitioned Hybrid Learning of Bayesian Network Structures\thanks{This work was supported by US NSF grant DMS-1952929}}

\author{Jireh Huang and Qing Zhou\thanks{UCLA Department of Statistics. Emails: jirehhuang@ucla.edu, zhou@stat.ucla.edu}}
\date{}
\maketitle

\begin{abstract}
We develop a novel hybrid method for Bayesian network structure learning called partitioned hybrid greedy search (pHGS), composed of three distinct yet compatible new algorithms: Partitioned PC (pPC) accelerates skeleton learning via a divide-and-conquer strategy, $p$-value adjacency thresholding (PATH) effectively accomplishes parameter tuning with a single execution, and hybrid greedy initialization (HGI) maximally utilizes constraint-based information to obtain a high-scoring and well-performing initial graph for greedy search. We establish structure learning consistency of our algorithms in the large-sample limit, and empirically validate our methods individually and collectively through extensive numerical comparisons. The combined merits of pPC and PATH achieve significant computational reductions compared to the PC algorithm without sacrificing the accuracy of estimated structures, and our generally applicable HGI strategy reliably improves the estimation structural accuracy of popular hybrid algorithms with negligible additional computational expense. Our empirical results demonstrate the superior empirical performance of pHGS against many state-of-the-art structure learning algorithms.
\end{abstract}

\textbf{Keywords:} Bayesian networks, structure learning, greedy search, PC algorithm

% \vspace{0.3in}

% \begin{center}
%     \textbf{Declarations}
% \end{center}

% \textbf{Funding:} This work was supported by US NSF grant DMS-1952929.

% \textbf{Conflicts of interest:} Not applicable.

% \textbf{Availability of data and material:} Not currently available.

% \textbf{Code availability:} Currently working on an \texttt{R} package to be made publicly available.

% \pagebreak

\section{Introduction}\label{sec:intro}

Bayesian networks are compact yet powerful graphical models that efficiently encode in their graphical structures probabilistic relationships amongst a large number of variables \citep{neapolitan2004}. Despite their utility for probabilistic inference, the problem of recovering from data the structure of the true underlying Bayesian network that governs a domain of variables is notoriously challenging \citep{chickering2004large}. The space of Bayesian network structures grows super-exponentially with the number of variables, severely limiting exhaustive evaluation of all structures and motivating decades of work in developing efficient algorithms for structure learning \citep{robinson1977, spirtes2000}. 

Generally, Bayesian network structure learning algorithms can be classified as one of the following three classes of algorithms. Constraint-based methods strategically test conditional independence relationships between pairs of variables, first determining the existence of edges before inferring orientations \citep{spirtes1991, meek1995}. In the score-based approach, heuristics are designed to optimize some scoring criterion that evaluates the goodness-of-fit of a proposed structure to the available data \citep{heckerman1995, chickering2002optimal, russell2009}. Finally, hybrid methods combine the two strategies, optimizing a score over a reduced space of structures restricted through a constraint-based approach \citep{tsamardinos2006max, gasse2014}. 

The PC algorithm %, named after its authors, 
\citep{spirtes1991} is often considered state-of-the-art amongst constraint-based methods for Bayesian network structure learning because of its polynomial complexity for sparse graphs and attractive theoretical properties \citep{kalisch2007}. Even with its favorable scaling, PC can quickly become unwieldy for large networks, motivating various developments to structure learning speed. Several works have contributed to accelerating its execution with various parallelization strategies, resulting in speed-ups ranging from up to ten times to over three orders of magnitude \citep{pcalg, le2016, madsen2017, scutari2017, zare2020}. However, these improvements are entirely feats of distributed processing implementation and are limited by the availability of required hardware. \cite{gu2020learning} proposed a hybrid framework for partitioned estimation of Bayesian networks called partition, estimation, and fusion (PEF) in the interest of distributing learning by adopting a divide-and-conquer strategy. Unfortunately, its application to the PC algorithm does not in general retain the completeness of the PC algorithm and is limited in its capacity for parallel processing.
Finally, none of these contributions tackle the practical problem that the performance of constraint-based algorithms can vary substantially with certain tuning parameters, potentially requiring multiple algorithm executions.

Prominent hybrid methods leverage the efficiency of constraint-based strategies to considerably reduce the space of Bayesian network models but sacrifice the asymptotic guarantees of constraint-based edge orientation for the generally superior empirical structural accuracy of restricted greedy search \citep{tsamardinos2006max}. This is characteristic of members of what we call the generalized sparse candidate (GSC) framework, named after the sparse candidate algorithm \citep{friedman1999}, in which a greedy search in the DAG space is executed from an empty graph restricted to a sparse set of candidate edges obtained through a constraint-based strategy. Hybrid algorithms belonging to GSC include max-min hill-climbing (MMHC) and hybrid hybrid parents and children (H2PC), which, despite their popularity and general regard for well-performance, are well-known to be lacking in asymptotic guarantees \citep{tsamardinos2006max, gasse2014}. While the adaptively restricted greedy equivalence search (ARGES) stands out as a hybrid framework with established consistency \citep{nandy2018}, our simulations suggest that ARGES can likewise empirically benefit from the developments in our work. In particular, both GSC and ARGES initialize their respective greedy searches with an empty graph and, to our knowledge, no principled and well-performing initialization strategy without assuming expert knowledge has been proposed. 

We propose an answer to these challenges by the development of the partitioned hybrid greedy search (pHGS) algorithm, a hybrid structure learning algorithm that can be considered the composition of three independent contributions to the computational efficiency, theoretical guarantees, and empirical performance of Bayesian network structure learning. In particular, pHGS accomplishes the following:
\begin{enumerate}[topsep=0.5em,itemsep=0.25em]

	\item Restricts the search space with our proposed partitioned PC (pPC) algorithm that improves on the efficiency of the PC algorithm while retaining its soundness and completeness and capacity for parallel processing;
	
	\item Mitigates the need for parameter tuning by automatically selecting the sparsity-controlling threshold of conditional independence tests with our $p$-value adjacency thresholding (PATH) algorithm that extends the accessibility of constraint-based consistency;
	
	\item Initializes the restricted greedy search with our hybrid greedy initialization (HGI) algorithm that elevates the asymptotic guarantees of existing hybrid algorithms such as members of the GSC framework to that of sound and complete constraint-based methods while  improving empirical performance.
	
\end{enumerate}

After reviewing relevant preliminaries in \autoref{sec:background}, the novel components of pHGS are organized in the remainder of this paper as follows. In \autoref{sec:ppc_path}, we develop the pPC algorithm which employs a partitioned estimation strategy to reduce the number of statistical tests required for the exhaustive conditional independence investigation in PC-like CPDAG learning. We additionally detail the PATH thresholding algorithm, which efficiently generates and selects from a set of CPDAG estimates with varying sparsity from a single execution of pPC (or PC) and extends the accessibility of classical asymptotic consistency results to more flexible parameter specification. We begin \autoref{sec:hybrid} with a brief review of score-based structure learning before developing HGI, a greedy initialization strategy which endears constraint-based edge orientation to the empirical setting with desirable theoretical guarantees. 

We empirically validate pPC, PATH, and HGI in \autoref{sec:results}, first independently and then collectively in the form of pHGS through an extensive simulation study. We show that pPC generally requires significantly fewer statistical calls as compared to PC, and that PATH effectively accomplishes the task of parameter tuning from a single algorithm execution with practically negligible computational expense. Compared to repeated executions of PC, the combined effect of pPC and PATH consistently achieves significant computational reductions without sacrificing (and indeed often improving on) estimation accuracy. We demonstrate the effectiveness of HGI on several instantiations of the GSC hybrid framework, and validate the holistic merits of pHGS against several popular structure learning algorithms. Though the focus of our paper is on the discrete case, we include succinct comments and results for our methods on high-dimensional Gaussian data.

\section{Background}\label{sec:background}

A \emph{graph} $\calG = (\bfV, \bfE)$ is a structure composed of a set of nodes $\bfV = \{1, \dots, p\}$, and a set of edges $\bfE$. For a pair of distinct nodes $i, j \in \bfV$, we encode an undirected edge between $i$ and $j$ in $\calG$ by an unordered connected pair $i \adjacent j \in \bfE$, and a directed edge from $i$ to $j$ in $\calG$ by an ordered pair $i \to j \in \bfE$.
A \emph{directed acyclic graph} (DAG) has only directed edges and is oriented such that there are no directed cycles in $\calG$. A DAG $\calG$ defines the structure of a \emph{Bayesian network} of a joint probability distribution $P$ of variables $\bfX$ corresponding to $\bfV$ if $P$ factorizes according to the structure of $\calG$: 
\begin{align}\label{eq:bn}
P(\bfX) = \prod_{i=1}^p P(X_i \mid \Pa{i}),
\end{align}
where $\Pa{i} = \{X_j : j \to i \in \bfE \}$ denotes the parents of $X_i$ according to $\calG$. In this paper, we may refer to a node $i \in \bfV$ and its corresponding variable $X_i \in \bfX$ interchangeably. In a causal DAG, $i \to j$ asserts that $i$ is a direct cause of $j$, whereas more generally, a DAG encodes in its structure a set of conditional independence statements between distinct variables according to the above factorization. For ease of notation, we let $\bfX_{\bfk} = \{X_k \in \bfX : k \in \bfk \}$ for $\bfk \subseteq \bfV$. 

This paper focuses on the setting in which $P$ is a discrete probability distribution, although many of the presented strategies are not limited to such a domain. Each variable $X_i$ probabilistically attains one of $r_i \geq 2$ states depending on the attained states of its parents $\Pa{i}$. The conditional probability distributions of the variables given each of their parent configurations are multinomial distributions. 

Let $\indep{X_i}{X_j}{\bfX_{\bfk}}$ denote that $X_i$ and $X_j$ are independent given conditioning set $\bfX_{\bfk} \subseteq \bfX \setminus \{X_i, X_j \}$ in $P$, and $\dsep{X_i}{X_j}{\bfX_{\bfk}}$ that $X_i$ and $X_j$ are d-separated by $\bfX_\bfk$ in $\calG$. The factorization \eqref{eq:bn} implies that $\calG$ and $P$ satisfy the \emph{(global) Markov condition}: for disjoint sets of variables $\bfA, \bfB, \bfC \subseteq \bfX$,
\begin{align}\label{eq:markov}
\dsep{\bfA}{\bfB}{\bfC} \Rightarrow \indep{\bfA}{\bfB}{\bfC}.
\end{align}

\subsection{Markov Equivalence}\label{sec:equivalence}

Multiple DAGs may encode the same set of d-separation statements and thus redundantly entail the same conditional independence statements. Such DAGs are said to be \emph{Markov equivalent}. Formally, two DAGs $\calG$ and $\calG^\prime$ are Markov equivalent if $\dsep{\bfA}{\bfB}{\bfC} \Leftrightarrow (\bfA \sep \bfB \!\mid\! \bfC)_{\mathcal{\calG}^\prime}$ for all mutually disjoint subsets $\bfA, \bfB, \bfC \subseteq \bfX$. We refer to Markov equivalent DAGs as simply \emph{equivalent} and belonging to the same \emph{equivalence class}. Given our distributional assumptions on $P$, equivalent DAGs are indistinguishable without background information or experimental data. As our interest lies in structure learning from observational data, the objective amounts to recovering the equivalence class of the underlying DAG. 

The \emph{skeleton} of a graph $\calG = (\bfV, \bfE)$ is the undirected graph obtained from replacing every connected node pair in $\calG$ with an undirected edge. A \emph{v-structure} is a triplet $i, j, k \in \bfV$ oriented $i \to k \gets j$ in $\calG$ with $i$ and $j$ not adjacent. Let the \emph{pattern} of $\calG$ be the \emph{partially directed acyclic graph} (PDAG) obtained by orienting all and only the v-structures of $\calG$ in its skeleton, leaving all remaining edges undirected. The following theorem was adapted from \cite{verma1991equivalence} to characterize equivalent DAGs. 

\begin{theorem}[\cite{meek1995}]\label{thm:equivalence}
Two DAGs are equivalent if and only if they have the same patterns.
\end{theorem}

Implied by \autoref{thm:equivalence} is the existence of \emph{compelled} and \emph{reversible} edges. An edge $i \to j$ in a DAG $\calG$ is {compelled} if it exists oriented as stated in every DAG in the equivalence class of $\calG$, whereas it is {reversible} if it is directed $j \to i$ in at least one DAG in the equivalence class of $\calG$. \cite{meek1995} detailed a set of sound and complete rules known as \emph{Meek's rules} (R1, R2, R3, and R4) that deterministically extend the pattern of a graph $\calG$ to its \emph{completed partially directed acyclic graph} (CPDAG), a PDAG featuring a directed edge for every compelled edge and an undirected edge for every reversible edge \citep{chickering2002learning}. As the unique representation of its equivalence class, the CPDAG is the structure of interest for structure learning methods in the observational setting.

\subsection{Faithfulness}\label{sec:faithful}

The global Markov property, as stated in \eqref{eq:markov}, defines an avenue for inference regarding the conditional independence relationships in $P$ according to information encoded in its Bayesian network structure $\calG$. As our interest is recovering $\calG$ from data generated from and thus sample estimates of probability distribution $P$, we require the assumption of faithfulness to infer the structure of $\calG$ from $P$. 

\begin{definition}[Faithfulness]\label{def:faithfulness}
A distribution $P$ and a DAG $\calG$ are said to be faithful to each other if all and only the conditional independence relations true in $P$ are entailed by the d-separation statements in $\calG$, i.e.
\begin{align*}
	\dsep{\bfA}{\bfB}{\bfC} \Leftrightarrow \indep{\bfA}{\bfB}{\bfC}.
\end{align*}
\end{definition}

Under faithfulness, we may say in such a case that $\bfA$ and $\bfB$ are separated by $\bfC$, regardless of whether we are referring to d-separation or conditional independence. If $P$ is faithful to $\calG$, then the existence of an edge between any distinct pair of nodes $i$ and $j$ can be necessarily and sufficiently determined by the nonexistence of a separation set of variables that render $i$ and $j$ conditionally independent in $P$. In particular,
\begin{align}\label{eq:exist}
	\text{$i, j \in \bfV$ are not connected in $\calG$} 
	~\Leftrightarrow~
	\exists \bfk \subseteq \bfV \setminus \{i, j \} \text{ such that } \indep{X_i}{X_j}{\bfX_{\bfk}}.
\end{align}

Throughout the development of our methodology, in what we call the population versions of procedures, we assume possession of all conditional independence information in $P$ denoted $\indepinfo$, thus having conditional independence oracles perfectly corresponding to d-separation. For inferring conditional independence from finite samples of discrete data $\calD$ in the sample counterparts, we use the popular $G^2$ log-likelihood ratio test of independence for empirical estimation of conditional independence in $P$ with some significance level threshold $\alpha$, denoting the $G^2$ test statistic for testing $\indep{X_i}{X_j}{\bfX_{\bfk}}$ as $G^2_{ij\mid \bfk}$ \citep{spirtes2000}. We briefly discuss the basic notation for and evaluation of the $G^2$ test in \autoref{sec:evaluations}, referring details and examples to \cite{neapolitan2004} 10.3.1.

\subsection{The PC Algorithm}\label{sec:pc}

The well-known PC algorithm \citep{spirtes1991}, named after its authors, is often considered the gold standard constraint-based structure learning method. The PC algorithm first efficiently estimates a skeleton, reducing the criterion stated in \eqref{eq:exist} by leveraging sparsity. Let $\bfN_i^\calG = \{X_j \in \bfX: \text{$i$ and $j$ are connected in $\calG$} \}$ be the \emph{neighbors}, or {adjacencies}, of node $i$ in a graph $\calG = (\bfV, \bfE)$.
If $\calG$ is a DAG, the following is evident from the Markov condition: 
\begin{align}\label{eq:exist2}
	\begin{split}
	\text{$i, j \in \bfV$ are not connected in $\calG$} 
		~\Leftrightarrow~
		&\exists \bfX_\bfk \subseteq {\bfN_i^\calG \setminus \{X_j\} } \text{ or } \exists \bfX_\bfk \subseteq {\bfN_j^\calG \setminus \{X_i\} } \\ 
		&\text{ such that } \indep{X_i}{X_j}{\bfX_{\bfk}}.
	\end{split}
\end{align}

\begin{algorithm}[h]
\noindent
\begin{minipage}{\textwidth}
\renewcommand*\footnoterule{}
\begin{savenotes}  % collect footnotes within algorithm environment
\caption{\texttt{PC-skeleton($\indepinfo$)} (PC-stable implementation; population version)}
\label{alg:pc}
\begin{algorithmic}[1]
\Require{conditional independence information $\indepinfo$}\footnote{As discussed, in the population versions of procedures we assume possession of conditional independence oracles. For finite-sample execution, we replace $\indepinfo$ with data samples $\calD$ from which conditional independence relationships are inferred using a consistent test and some threshold $\alpha$.}
\Ensure undirected graph $\calG$ 
\State{form the complete undirected graph $\calG = (\bfV, \bfE)$ over nodes $\bfV = \{k : X_k \in \bfX \}$}\label{algl:form}
\State{initialize $\bfS = \emptyset$}\label{algl:initialize_s}
\State{initialize $l = 0$}\label{algl:l}
\Repeat\label{algl:pc_repeat}
\State{set $\calG^\prime = \calG$ to fix adjacencies}\label{algl:pc_fix}
\ForAll{unordered node pairs $i, j$ adjacent in $\calG^\prime$ with $\abs{\bfN_i^{\calG^\prime} \setminus \{X_j\} }$ or $\abs{\bfN_j^{\calG^\prime} \setminus \{X_i\} } \geq l$}\label{algl:distinct_pairs}
\ForAll{unique subsets $\bfX_\bfk \subseteq \bfN_i^{\calG^\prime} \setminus \{X_j\}$ and $\bfX_\bfk \subseteq \bfN_j^{\calG^\prime} \setminus \{X_i\}$ of size $\abs{\bfk} = l$}\label{algl:subsets}
\If{$\indep{X_i}{X_j}{\bfX_{\bfk}}$}\label{algl:pc_update}
\State store separation set $\bfS(i, j) = \bfS(j, i) = \bfk$\label{algl:sep}
\State disconnect $i$ and $j$ in $\calG$ and continue to the next pair of nodes
\EndIf\label{algl:end_pc_update}
\EndFor
\EndFor\label{algl:end_distinct_pairs}
\State{$l = l + 1$}
\Until{there is no node $i$ with $\abs{\bfN_i^{\calG}} - 1 \geq l$, or $l > m$ for some user-specified $m$}\label{algl:pc_until}
\end{algorithmic}
\end{savenotes}  % footnotes are displayed at the end of the savenotes environment
\end{minipage}
\end{algorithm}

\noindent For easy reference in our algorithm description, we detail an implementation of the skeleton estimation step of the PC algorithm known as PC-stable in \autoref{alg:pc} \citep{colombo14a}. The key difference from the original PC algorithm is that in line \ref{algl:pc_fix}, the adjacencies are fixed in $\calG^\prime$ such that the considerations of adjacent node pairs within the outermost loop (lines \ref{algl:pc_repeat}-\ref{algl:pc_until}) become order-independent and thus executable in parallel. We further discuss parallel execution of the PC algorithm in \autoref{sec:skeleton}. Note that for every node $i$, $\bfN_i^\calG \subseteq \bfN_i^{\calG^\prime}$ for $\bfN_i^{\calG^\prime}$ in any stage in \autoref{alg:pc}, preserving the general design of the original PC skeleton learning method by ensuring the exhaustive investigation of \eqref{eq:exist2} and thus retaining its theoretical properties. Hereafter, when we discuss the PC algorithm, we refer to the PC-stable implementation.

After determining the skeleton of a DAG $\calG$, knowledge about the conditional independence relationships between variables (namely, the accrued separation sets $\bfS$) can be used to detect the existence of v-structures and orient the skeleton to the pattern of $\calG$. Recovery of the CPDAG of $\calG$ can then be achieved by repeated application of Meek's rules \citep{meek1995}. This process, which we refer to as \texttt{skel-to-cpdag} (\autoref{alg:cpdag} in \autoref{sec:orient}), is guaranteed to orient the skeleton of a DAG $\calG$ to its CPDAG given accurate conditional independence information entailed by $\calG$. For details regarding constraint-based edge orientation, see \autoref{sec:orient}.

The complete PC(-stable) algorithm consists of skeleton estimation according to \autoref{alg:pc} followed by edge orientation according to \autoref{alg:cpdag}, and is well-known to be sound and complete for CPDAG estimation \citep{kalisch2007, colombo14a}.

\section{The pPC and PATH Algorithms}\label{sec:ppc_path}

In constraint-based methods, the computational expense of edge orientation has been noted to be generally insignificant compared to that of skeleton estimation (see \autoref{sec:ppc}) \citep{chickering2002learning, madsen2017}. As such, we develop the partitioned PC algorithm (pPC) to reduce the computational expense of skeleton estimation by imposing a partitioned ordering to the conditional independence tests. Similarly, we propose the $p$-value adjacency thresholding (PATH) algorithm that effectively accomplishes the task of parameter tuning by efficiently generating a solution path of estimates from a single execution of pPC or PC.

\subsection{The Partitioned PC Algorithm}\label{sec:ppc}

The pPC algorithm improves on the already desirable efficiency of the PC algorithm while retaining its attractive theoretical properties and empirical structure learning accuracy. The structure follows similarly to the partition, estimation, and fusion (PEF) strategy applied to the PC algorithm in \cite{gu2020learning}. We develop improvements and computational exploits to further increase performance, formulate pPC to retain soundness and completeness, and propose adaptations to address the challenges of learning the structure of discrete Bayesian networks. 

The intuition motivating a partitioned strategy is that any structure learning algorithm that scales worse than linear to $p$ will be able to estimate $\kappa > 1$ subgraphs for node clusters that partition the $p$ nodes faster than a single graph on all nodes. If the $p$ nodes can be reliably partitioned such that the connectivity between clusters is weak relative to within clusters, then we can expect that there will not be many false positive edges (as a result of causal insufficiency) within subgraphs. As a consequence, the adjacencies are expected to be relatively well-estimated, providing a selective candidate set of neighbors to screen the edges between subgraphs. Coupled with the assumed weak connectivity between clusters, the process of determining the existence of edges amongst clusters is expected to be efficient.

The pPC algorithm estimation process proceeds as follows. We partition the $p$ nodes into $\kappa$ clusters using a version of the modified hierarchical clustering algorithm proposed in \cite{gu2020learning} applied with a normalized discrete distance metric, additionally blacklisting marginally independent node pairs. We then apply the PC algorithm to estimate edges within clusters, and filter and refine edges between nodes in different clusters. Finally, we achieve completeness by applying a reduced PC algorithm before orienting the edges.

\subsubsection{Clustering}\label{sec:cluster}

As previously motivated, the task of obtaining an effective partition of the nodes is crucial for the success of the skeleton learning. A partition with many clusters $\kappa$ is desirable for greatest computational benefit in subgraph estimation, but each cluster must be substantive so as to minimally violate causal sufficiency. To accomplish this, the distances between nodes are measured by a normalized mutual information distance metric, and the target number of clusters and initial clusters are chosen adaptively. 

Mutual information, denoted $\I(X_i, X_j)$, serves as a similarity measure between discrete random variables $X_i$ and $X_j$ and may be interpreted as the Kullback-Leibler divergence between the joint probability distribution and the product of their marginals. We obtain a distance measure by inverting the pairwise mutual information after normalizing using the joint entropy ${{\H}}(X_i, X_j)$. In particular, the distance between each pair of variables $X_i$ and $X_j$ is defined as
\begin{align}\label{eq:distance}
	d_{ij} = %d_{ji} = 
	1 - \frac{{\I}(X_i, X_j)}{{\H}(X_i, X_j)} \in [0, 1].
\end{align}
The proposed distance $d_{ij}$ is a metric in the strict sense as shown by \cite{kraskov2005}, meaning it is symmetric, non-negative, bounded, and satisfies the triangle inequality. In practice, we compute the empirical quantities of the mutual information and joint entropy $\hat{\I}$ and $\hat{\H}$ (see \autoref{sec:evaluations}).

Given our distance matrix $D = (d_{ij})_{p \times p}$, we apply Algorithm 1 in \cite{gu2020learning} with average linkage to determine a cut $l$ for the agglomerative hierarchical clustering of the $p$ nodes. Succinctly described, we choose the highest cut such that the resulting cluster consists of the greatest number of large clusters, defined as node clusters of at least size $0.05p$ according to a loose suggestion by \cite{hartigan1981}. We then merge clusters of size less than $0.05p$ with other small clusters or into large clusters sequentially, ordered by average linkage, until every cluster is a large cluster. For further details regarding the algorithm, we refer to the original paper. The clustering step partitions the $p$ nodes into $\kappa$ clusters, returning the cluster labels $\bfc = \{c_1, \dots, c_p\}$, with $c_i \in \{1, \dots, \kappa\}$ denoting the cluster label of node $i$.

While the pairwise computation of both the mutual information and the joint entropy may seem expensive for the purpose of obtaining a partition, we take advantage of two exploits to accomplish this economically. Observing that $\I(X_i, X_j) =  \H(X_i) + \H(X_j) - \H(X_i, X_j)$, we need only compute the marginal entropies $\H(X_i) = \I(X_i, X_i)$ to derive the joint entropies from the pairwise mutual information, a reduction from $p(p-1)$ computations to $p(p+1) / 2$. Further noting that the discrete unconditional $G^2$ test statistic for investigating the marginal independence between $X_i$ and $X_j$ is computed as $G^2_{ij} = 2n \cdot \hat{\I} (X_i, X_j)$, an initial edge screening can easily be obtained through the evaluation
\begin{align}\label{eq:marginal}
	\begin{split}
	\rmPr(\chi^2_f > 2n \cdot \hat{\I} (X_i, X_j)) > \alpha ~&\Rightarrow~ (X_i \sep X_j)_P ~\\&\Rightarrow~ \text{blacklist the edge $i \adjacent j$}.
	\end{split}
\end{align}
This effectively accomplishes the empty conditioning set ($l = 0$) testing step of the PC algorithm by separating all marginally independent distinct node pairs (\autoref{alg:ppc} line \ref{algl:pc_within}).

\subsubsection{Partitioned Skeleton Estimation}\label{sec:skeleton}

We now apply the PC algorithm skeleton learning phase (\autoref{alg:pc}) to estimate $\kappa$ disconnected undirected subgraphs according to the partition obtained in the clustering step. Practically, independently applying the PC algorithm to each node cluster benefits from at most $\kappa$ processors in parallel processing. Furthermore, the speed-up is limited by the longest estimation runtime, usually corresponding to the largest node cluster. In contrast, the design of the PC-stable implementation (\autoref{alg:pc}) by \cite{colombo14a} allows for parallel investigation of adjacent node pairs in lines \ref{algl:distinct_pairs}-\ref{algl:end_distinct_pairs}, provided that updating the graph estimate is deferred to a synchronization step between repetitions. Several contributions and implementations exist for this approach, referred to as {vertical parallelization}, which addresses the case where the number of variables $p$ is large \citep{pcalg, le2016, scutari2017, zare2020}. Alternatively, a {horizontal parallelization} approach parallelizes across data observations and is preferred when the sample size $n$ is large \citep{madsen2017}. In these parallelization paradigms, due to the large number of distributed tasks, the number of utilizable computing processors is not practically limited, and the computational load is reasonably expected to be evenly distributed. To take advantage of these developments in parallelizing the PC algorithm, we estimate subgraphs within the node clusters by executing \autoref{alg:pc} with the following modifications: (i) form the initial complete undirected graph by only connecting nodes within clusters, (ii) delete edges according to \eqref{eq:marginal}, and (iii) begin investigating candidate conditioning sets of size $l = 1$. The result is an undirected graph $\calG$ on $\bfV$ consisting of $\kappa$ disconnected subgraphs.

At this stage, given a good partition, we expect the node adjacencies to be relatively well-estimated, with the exception of extraneous edge connections within clusters due to the violation of causal sufficiency and missing edge connections between clusters that are disconnected by the partition. Recall from \autoref{sec:cluster} that many pairs, including those between clusters, were removed from consideration via the initial marginal independence filtering according to \eqref{eq:marginal}. We further refine the pairs between clusters through a two-step screening process, our strategy being similar to Algorithm 2 designed by \cite{gu2020learning}, with modifications made for discrete structure learning. Note that this is where we anticipate to derive the most computational advantage over the PC algorithm. Assuming a block structure was successfully detected, we aim to circumvent the lower order conditional independence tests in our investigation of \eqref{eq:exist2} by separating as many between pairs as possible with the currently estimated adjacencies. 

The proposed between cluster screening process is summarized in the following operations to the currently estimated skeleton $\calG = (\bfV, \bfE)$:
\begin{align}
	\label{eq:screen1}
	\bfE &\gets \bfE ~\cup~ \left\{i \adjacent j : c_i \neq c_j \text{, } (X_i \!\not\sep\! X_j)_P \text{, and } \notindep{X_i}{X_j}{\bfN_{i}^\calG \cup \bfN_{j}^\calG} \right\}, \\
	\label{eq:screen2}
	\bfE &\gets \bfE ~\setminus~ \left\{i \adjacent j \in \bfE : c_i \neq c_j \text{, and } \indep{X_i}{X_j}{\bfN_{i}^{\calG} \setminus \{X_j\}} \text{ or } \indep{X_i}{X_j}{\bfN_{j}^{\calG} \setminus \{X_i\}} \right\},
\end{align}
where $c_i$ is the cluster label of $X_i$. The first screen \eqref{eq:screen1} constructively connects the between cluster edges that are dependent marginally, as assessed according to \eqref{eq:marginal}, as well as conditioned on the union of the neighbor sets. With the addition of edges between clusters, the second screen \eqref{eq:screen2} disconnects pairs that are separated by the newly updated adjacencies. As in \autoref{alg:pc}, we fix adjacencies to retain the capacity for parallel investigation of the considered node pairs. Note that every between-cluster edge that is present in the underlying DAG will be connected by \eqref{eq:screen1}, and \eqref{eq:screen2} can only prune false positives. Thus, after this step every node pair will have been considered and, in the population setting, every truly connected edge in the underlying DAG will be connected in $\calG$.

\begin{algorithm}[!b]
\noindent
\begin{minipage}{\textwidth}
\renewcommand*\footnoterule{}  % switch off footnote line locally
\begin{savenotes}  % collect footnotes within algorithm environment
\caption{\texttt{pPC($\indepinfo$)} (population version)}\label{alg:ppc}
\begin{algorithmic}[1]
\Require{conditional independence information $\indepinfo$}
\Ensure{CPDAG estimate $\calG$}
\Statex \textbf{\emph{partitioning}}
\State assign cluster labels $\bfc$ by partitioning all nodes into $\kappa$ clusters as in \autoref{sec:cluster}
\Statex \textbf{\emph{within cluster edge estimation}}
\State execute $\calG \gets \texttt{PC-skeleton($\indepinfo$)}$ (\autoref{alg:pc}), with the following modifications: 
\begin{itemize}[label={},leftmargin=1em]
	\item line \ref{algl:form}: form undirected graph by only connecting node pairs $i, j \in \bfV$ with $c_i = c_j$
	\item line \ref{algl:l}: delete marginally independent pairs from $\calG$ according to \eqref{eq:marginal} and initialize $l = 1$
\end{itemize} \label{algl:pc_within}
\Statex \textbf{\emph{between cluster edge screening}}
\State{initialize $\calG^\prime = \calG$ to fix adjacencies}\label{algl:between_fix1}
\ForAll {unordered node pairs $i, j$ such that $c_i \neq c_j$ and $(X_i \not\sep X_j)_P$}
\If{$\notindep{X_i}{X_j}{\bfN_{i}^{\calG^\prime} \cup \bfN_{j}^{\calG^\prime}}$} \label{algl:ppc_eval1}
\State connect $i$ and $j$ in $\calG$
\Else
\State store separation set in $\bfS(i, j) = \bfS(j, i)$
\EndIf
\EndFor\label{algl:end_screen1}
\State{set $\calG^\prime = \calG$ to update fixed adjacencies}\label{algl:between_fix2}
\ForAll {unordered node pairs $i,j$ adjacent in $\calG$ with $c_i \neq c_j$}
\If{$\indep{X_i}{X_j}{\bfN_{i}^{\calG^\prime} \setminus \{X_j\}}$ or $\indep{X_i}{X_j}{\bfN_{j}^{\calG^\prime} \setminus \{X_i\}}$} \label{algl:ppc_eval2}
\State store separation set in $\bfS(i, j) = \bfS(j, i)$
\State disconnect $i$ and $j$ in $\calG$
\EndIf
\EndFor\label{algl:end_screen2}
\Statex \textbf{\emph{complete edge estimation}}
\State execute $\calG \gets\texttt{PC-skeleton($\indepinfo$)}$ (\autoref{alg:pc}), continuing from $\calG$ and restricting to \eqref{eq:criteria} with the following replacements:\label{algl:ppc_skeleton}
\begin{itemize}[label={},leftmargin=1em]
	\item line \ref{algl:form}: initialize undirected graph with current estimate $\calG$
	\item line \ref{algl:initialize_s}: omit this line
	\item line \ref{algl:l}: initialize $l = 1$
	\item line \ref{algl:subsets}: \textbf{for all} unique subsets ${\bfX_\bfk} \subseteq \bfN_{i}^{\calG^\prime} \setminus \{X_j\}$ and $\bfX_\bfk \subseteq \bfN_{j}^{\calG^\prime} \setminus \{X_i\}$ satisfying $\abs{\bfk} = l$ and criteria \eqref{eq:criteria} \textbf{do}
\end{itemize}
\textbf{\emph{edge orientation}}
\State execute $\calG \gets \texttt{skel-to-cpdag($\calG, \bfS$)}$ (\autoref{alg:cpdag}) to orient $\calG$ \label{algl:ppc_orient}
\end{algorithmic}
\end{savenotes}  % footnotes are displayed at the end of the savenotes environment
\end{minipage}
\end{algorithm}

\begin{remark}\label{rmk:test}
The formulation of the discrete $G^2$ tests of independence require enumerating and counting across all conditioning variable configurations. An unavoidable consequence is that the computational complexity and memory requirement increase dramatically with increasing conditioning variables, especially when they have a large number of discrete levels. As such, it is of practical interest to restrict the size of conditioning sets to some user-specified $m > 0$. Thus, for the evaluations in \eqref{eq:screen1} and \eqref{eq:screen2} (\autoref{alg:ppc} lines \ref{algl:ppc_eval1} and \ref{algl:ppc_eval2}), if for example $\abs{\bfN_i^\calG \cup \bfN_j^\calG} > m$, we instead investigate the unique sets ${\bfX_\bfk} \subseteq \bfN_{i}^\calG \cup \bfN_{j}^\calG$ such that $\abs{\bfk} = m$.
\end{remark}

However, it is important to note that the constraint-based criterion for edge existence \eqref{eq:exist2} has not yet been fully investigated for all node pairs. The tests for edges within clusters only considered conditioning sets consisting of nodes within the same clusters, and tests for edges between clusters were limited to the empty set, $\bfN_i^\calG \cup \bfN_j^\calG$, $\bfN_i^\calG \setminus \{X_j\}$, and $\bfN_j^\calG \setminus \{X_i\}$. In particular, for each remaining adjacent node pair $i \adjacent j \in \bfE$, the possible separation sets that have not been evaluated are limited to either of the following cases, if any:
\begin{align}
	\begin{split}
	\label{eq:criteria}
	&\text{(a) $c_i = c_j$ (within clusters), sets $\bfk$ where $\exists k \in \bfk$ such that $c_k \neq c_i$;} \\
	&\text{(b) $c_i \neq c_j$ (between clusters), sets $\bfk \neq \emptyset$ not defined in \eqref{eq:screen1} or \eqref{eq:screen2}.}
	\end{split}
\end{align}
The most straightforward continuation to achieve completeness in our partitioned skeleton learning process would be to exhaustively evaluate the dependence of the remaining connected node pairs in $\calG$ conditioned on the remaining conditioning sets. We accomplish this by restarting the PC algorithm on the current skeleton, evaluating independence conditioned on sets restricted to criteria \eqref{eq:criteria}, before finally orienting the resulting skeleton to a CPDAG with \texttt{skel-to-cpdag} (\autoref{alg:cpdag}) to complete structure learning. 

The resulting pPC algorithm is detailed in \autoref{alg:ppc}, and an example of its execution is illustrated in \autoref{fig:ppc_eg}. The pPC algorithm first estimates the edges within the clusters to obtain $\kappa = 3$ disconnected subgraphs (\autoref{fig:ppc_eg_a}), and then constructively screens the edges between clusters according to \eqref{eq:screen1}, which recovers all true positives but also connects some false positive edges (\autoref{fig:ppc_eg_b}). Two false positive edges $X_6 \adjacent X_8$ and $X_8 \adjacent X_9$, are pruned by the second screening of edges between clusters \eqref{eq:screen2}, but $X_6 \adjacent X_9$ cannot be separated by either of the complete neighbor sets since they both include $X_{12}$, thus requiring $\indep{X_6}{X_9}{X_3}$ to be investigated in line \ref{algl:ppc_skeleton} (\autoref{fig:ppc_eg_c}).

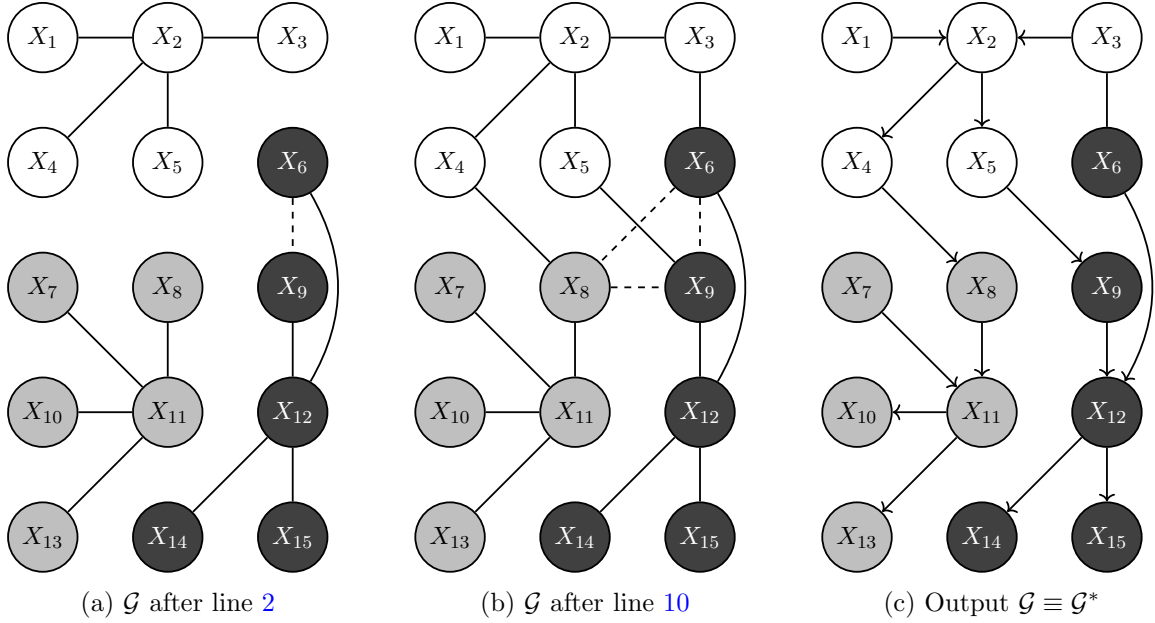
\begin{figure}[t]
\begin{subfigure}[b]{0.32\textwidth}
\centering
\resizebox{0.9\linewidth}{!}{
\noindent\begin{tikzpicture}[	thick, node distance={5em},
							one/.style = {draw, circle, minimum size=2.78em}, 
							two/.style={draw, circle, fill=lightgray, minimum size=2.78em},
							three/.style={draw, circle, fill=darkgray, text=white, minimum size=2.78em}] 
\node[one] (1) {$X_1$};
\node[one] (2) [right of=1] {$X_2$};
\node[one] (3) [right of=2] {$X_3$}; 
\node[one] (4) [below of=1] {$X_4$};
\node[one] (5) [right of=4] {$X_5$};
\node[three] (6) [right of=5] {$X_6$};
\node[two] (7) [below of=4] {$X_7$};
\node[two] (8) [right of=7] {$X_8$};
\node[three] (9) [right of=8] {$X_9$};
\node[two] (10) [below of=7] {$X_{10}$};
\node[two] (11) [right of=10] {$X_{11}$};
\node[three] (12) [right of=11] {$X_{12}$};
\node[two] (13) [below of=10] {$X_{13}$};
\node[three] (14) [right of=13] {$X_{14}$};
\node[three] (15) [right of=14] {$X_{15}$};
\draw (1) -- (2);
\draw (3) -- (2);
\draw (2) -- (4);
\draw (2) -- (5);
\draw (7) -- (11);
\draw (8) -- (11);
\draw (11) -- (10);
\draw (11) -- (13);
\draw (6) to [out=300,in=60,looseness=1] (12);
\draw (9) -- (12);
\draw (12) -- (14);
\draw (12) -- (15);
% false negatives
%\draw (4) -- (8);
%\draw (3) -- (6);
%\draw (5) -- (9);
% false positives
\draw[dashed] (6) -- (9);
\end{tikzpicture}
}
\caption{$\calG$ after line \ref{algl:pc_within}}
\label{fig:ppc_eg_a}
\end{subfigure}
\begin{subfigure}[b]{0.32\textwidth}
\centering
\resizebox{0.9\linewidth}{!}{
\noindent\begin{tikzpicture}[	thick, node distance={5em},
							one/.style = {draw, circle, minimum size=2.78em}, 
							two/.style={draw, circle, fill=lightgray, minimum size=2.78em},
							three/.style={draw, circle, fill=darkgray, text=white, minimum size=2.78em}] 
\node[one] (1) {$X_1$};
\node[one] (2) [right of=1] {$X_2$};
\node[one] (3) [right of=2] {$X_3$}; 
\node[one] (4) [below of=1] {$X_4$};
\node[one] (5) [right of=4] {$X_5$};
\node[three] (6) [right of=5] {$X_6$};
\node[two] (7) [below of=4] {$X_7$};
\node[two] (8) [right of=7] {$X_8$};
\node[three] (9) [right of=8] {$X_9$};
\node[two] (10) [below of=7] {$X_{10}$};
\node[two] (11) [right of=10] {$X_{11}$};
\node[three] (12) [right of=11] {$X_{12}$};
\node[two] (13) [below of=10] {$X_{13}$};
\node[three] (14) [right of=13] {$X_{14}$};
\node[three] (15) [right of=14] {$X_{15}$};
\draw (1) -- (2);
\draw (3) -- (2);
\draw (2) -- (4);
\draw (2) -- (5);
\draw (7) -- (11);
\draw (8) -- (11);
\draw (11) -- (10);
\draw (11) -- (13);
\draw (6) to [out=300,in=60,looseness=1] (12);
\draw (9) -- (12);
\draw (12) -- (14);
\draw (12) -- (15);
% true between
\draw (4) -- (8);
\draw (3) -- (6);
\draw (5) -- (9);
% between false positives
\draw[dashed] (6) -- (8);
\draw[dashed] (6) -- (9);
\draw[dashed] (8) -- (9);
\end{tikzpicture}
}
\caption{$\calG$ after line \ref{algl:end_screen1}}
\label{fig:ppc_eg_b}
\end{subfigure}
\begin{subfigure}[b]{0.32\textwidth}
\centering
\resizebox{0.9\linewidth}{!}{
\noindent\begin{tikzpicture}[	thick, node distance={5em},
							one/.style = {draw, circle, minimum size=2.78em}, 
							two/.style={draw, circle, fill=lightgray, minimum size=2.78em},
							three/.style={draw, circle, fill=darkgray, text=white, minimum size=2.78em}] 
\node[one] (1) {$X_1$};
\node[one] (2) [right of=1] {$X_2$};
\node[one] (3) [right of=2] {$X_3$}; 
\node[one] (4) [below of=1] {$X_4$};
\node[one] (5) [right of=4] {$X_5$};
\node[three] (6) [right of=5] {$X_6$};
\node[two] (7) [below of=4] {$X_7$};
\node[two] (8) [right of=7] {$X_8$};
\node[three] (9) [right of=8] {$X_9$};
\node[two] (10) [below of=7] {$X_{10}$};
\node[two] (11) [right of=10] {$X_{11}$};
\node[three] (12) [right of=11] {$X_{12}$};
\node[two] (13) [below of=10] {$X_{13}$};
\node[three] (14) [right of=13] {$X_{14}$};
\node[three] (15) [right of=14] {$X_{15}$};
\draw[->] (1) -- (2);
\draw[->] (3) -- (2);
\draw[->] (2) -- (4);
\draw[->] (2) -- (5);
\draw (3) -- (6);
\draw[->] (4) -- (8);
\draw[->] (5) -- (9);
\draw[->] (7) -- (11);
\draw[->] (8) -- (11);
\draw[->] (11) -- (10);
\draw[->] (11) -- (13);
\draw[->] (6) to [out=300,in=60,looseness=1] (12);
\draw[->] (9) -- (12);
\draw[->] (12) -- (14);
\draw[->] (12) -- (15);
\end{tikzpicture}
}
\caption{Output $\calG \equiv \calG^*$}
\label{fig:ppc_eg_c}
\end{subfigure}
\caption{Example of various stages of pPC (\autoref{alg:ppc}). Node shading denotes clusters, and dashed lines represent false positive edges. The CPDAG of the underlying DAG is $\calG^*$ in (c).}
\label{fig:ppc_eg}
\end{figure}

The pPC algorithm (\autoref{alg:ppc}) is sound and complete, formalized in \autoref{thm:ppc}, which we prove in \autoref{sec:proofs}.

\begin{theorem}\label{thm:ppc}
	Suppose that probability distribution $P$ and DAG $\calG^*$ are faithful to each other. Then given conditional independence oracles, the output of the partitioned PC algorithm (\autoref{alg:ppc}) is the CPDAG that represents the equivalence class of $\calG^*$.
\end{theorem}

Its implication is the asymptotic consistency of pPC for fixed $p$ as $n\to\infty$, given a consistent conditional independence test such as the $G^2$ test \citep{cressie1989}. Note that while the computational savings may depend heavily on the quality of the partitioning, \autoref{thm:ppc} holds regardless of the obtained clusters.

\subsection{\protect\boldmath $p$-value Adjacency Thresholding}\label{sec:path}

Despite the attractive theoretical properties of algorithms such as pPC and PC, it is well-known that in practice, constraint-based algorithms suffer from the multiple testing problem, a challenge exacerbated when $p$ is large \citep{spirtes2010}. The effect of individual errors can compound throughout the conditional independence testing process, leading to erroneous inferences regarding both the existence and orientation of edges \citep{koller2009, spirtes2010}. In addition to deteriorated quality of structures estimated, mistakes in conditional independence inferences can result in invalid PDAG estimates that do not admit a consistent extension (see \autoref{sec:equivalence}). In practical applications, the choice of conditional independence test threshold $\alpha$ can significantly control the sparsity and quality of the resulting estimate. To the best of our knowledge, proposed theoretical thresholds depend on unknown quantities and are not practically informative, such as in \cite{kalisch2007}. Empirically, the optimal choice of $\alpha$ varies depending on factors such as the sample size and the structure and parameters of the underlying Bayesian network, and no universally well-performing value is known. 

We propose the p-value adjacency thresholding (PATH) algorithm to generate and select from a CPDAG solution path across various values of $\alpha$ from a single execution of the pPC algorithm, or indeed from any constraint-based structure learning algorithm that is able to obtain the following. Define the maximum p-values $\Phi=(\Phi_{ij})$ such that $\Phi_{ij}$ is the maximum p-value obtained by the $G^2$ test of independence between $i$ and $j$ over all conditioning sets $\bfk \in \calK_{ij}$ visited in the algorithm, as well as the corresponding separation sets, thus extending the definition of $\bfS$. That is, for all distinct node pairs $i, j$, 
\begin{align}
	\begin{split}
	\label{eq:sep}
	\Phi_{ij} = \Phi_{ji} &\coloneqq \max_{\bfk \in \calK_{ij}} \rmPr \left( \chi^2_f > G^2_{ij|\bfk} \right),  \\
	\bfS(i, j) = \bfS(j, i) &\coloneqq \argmax_{\bfk \in \calK_{ij}} \rmPr \left(\chi^2_f > G^2_{ij|\bfk} \right),
	\end{split}
\end{align}
for degrees of freedom $f$ corresponding to the test of independence between $i$ and $j$ conditioned on $\bfk$. For a connected node pair $i \adjacent j$, $\bfS(i, j)$ may be considered the conditioning set closest to separating $i$ and $j$, and $\Phi_{ij}$ measures how close. 

The process itself is straightforward: for a sequence of significance levels $\{\alpha^{(t)}\}$, we obtain updated PDAG estimates $\calG^{(t)}$ by thresholding the maximum achieved p-values $\Phi$ to obtain skeleton estimates with edge sets $\bfE^{(t)} = \{i \adjacent j: \Phi_{ij} \leq \alpha^{(t)} \}$ and then orienting them to CPDAGs according \texttt{skel-to-cpdag} (\autoref{alg:cpdag}) with the corresponding separation information $\bfS^{(t)} = \{\bfS(i, j) : \Phi_{ij} > \alpha^{(t)} \}$. The quality of the estimates are then evaluated by a scoring criterion and the highest-scoring network is returned. In what follows, we develop the choice of the threshold values and present the strategy for estimate generation and selection. 

We begin with a graph estimate obtained by executing the pPC algorithm with some maximal threshold value $\alpha$. The goal is to start with the densest graph so that the elements of $\Phi$ and $\bfS$ in \eqref{eq:sep} are maximized over a larger number of visited conditioning sets $\abs{\calK}$. We generate a sequence of $\tau$ values decreasing from $\alpha^{(1)} \coloneqq \alpha$ to some minimum threshold value $\alpha^{(\tau)}$. This sequence may be incremented according to some linear or log-linear scale, but we choose to achieve maximal difference in sparsity amongst estimates by utilizing the information in $\Phi$. Given each $\alpha^{(t)}$ corresponding to estimate $\calG^{(t)} = (\bfV, \bfE^{(t)})$, we choose $\alpha^{(t+1)}$ such that 
\begin{align}\label{eq:edge_diff}
	\abs{\bfE^{(t)}} - \abs{\bfE^{(t+1)}} \approx \frac{\abs{\bfE^{(1)}} - \abs{\bfE^{(\tau)}}}{\tau - 1}.
\end{align}
Noting that $\abs{\bfE^{(t)}} = \sum_{i < j} \ind{\Phi_{ij} \leq \alpha^{(t)}}$ for indicator function $\ind{\cdot}$, it is easy to see that the sequence $\alpha^{(1)}, \dots, \alpha^{(\tau)}$ can be straightforwardly obtained using the order statistics of the elements of $\Phi$. 

Once a solution path of CPDAG estimates $\{ \calG^{(t)} : t \in \{1, \dots, \tau \} \}$ is obtained, we select the highest quality estimate by means of score-based selection. The \emph{Bayesian information criterion} (BIC) is a penalized log-likelihood score derived from the asymptotic behavior of Bayesian network models, with established consistency \citep{schwarz1978}. The formulation of the BIC score makes clear its \emph{score decomposability} (see \eqref{eq:penll}); that is, the score $\phi(\calG, \calD)$ can be computed as the sum of the scores of the individual variables with respect to their parents in $\calG$: $\phi (\calG, \calD) = \sum_{i=1}^p \phi(X_i, \Pa{i})$. The BIC score is additionally \emph{score equivalent}, evaluating all Markov equivalent DAGs as of identical quality, with a higher value indicating a better fit to the data. 

Due to score equivalence, it is sufficient to evaluate each CPDAG $\calG^{(t)}$ with any arbitrary DAG $\tilde{\calG}^{(t)}$ in its equivalence class, called a \emph{consistent extension} of $\calG^{(t)}$. \cite{dor1992simple} proposed a simple algorithm, which we denote \texttt{pdag-to-dag}, that obtains such an extension by orienting the undirected (reversible) edges in $\calG^{(t)}$ without inducing directed cycles or introducing new v-structures, and is guaranteed to succeed if a consistent extension exists (further discussed in \autoref{sec:hgi}). After obtaining these DAG extensions, in practice, score decomposability can be leveraged to avoid scoring $p$ nodes for $\tau$ estimates by setting $\Delta^{(1)} = 0$ and computing the score differences between estimates $\Delta^{(t)} = \phi(\tilde{\calG}^{(t)}, \calD) - \phi(\tilde{\calG}^{(t-1)}, \calD)$ for $t = 2, \dots, \tau$, caching computed node score to avoid redundant computations.
The best solution can then be straightforwardly obtained according to
	\begin{align*}
		t^* 	&= \argmax_{t \in \{1, \dots, \tau\}} \phi(\calG^{(t)}, \calD) 
			=\argmax_{t \in \{1, \dots, \tau\}} \sum_{r = 1}^t \Delta^{(r)}.
	\end{align*}
We detail the PATH solution path strategy in \autoref{alg:path}.

\begin{algorithm}
\caption{\texttt{PATH($\Phi$, $\bfS$, $\calD$, $\tau$, $\alpha^{(\tau)}$)} (sample version)}
\label{alg:path}
\begin{algorithmic}[1]
\Require{maximum p-values $\Phi$, separation sets $\bfS$, data $\calD$, number of estimates $\tau$, minimum threshold $\alpha^{(\tau)}$}
\Ensure{CPDAG $\calG^{(t^*)}$, CPDAG solution path $\{\calG^{(t)} : t \in \{1, \dots, \tau\}\}$}
\State generate a decreasing sequence of $\tau$ values $\{\alpha^{(t)}\}$ from $\alpha^{(1)} \coloneqq \max_{i, j} \Phi_{ij}$ to $\alpha^{(\tau)}$ by \eqref{eq:edge_diff}\label{algl:alpha}
\For {$t = 1, 2,  \dots, \tau$}
\State obtain skeleton $\calG^{(t)} = (\bfV, \bfE^{(t)})$ by thresholding: $\bfE^{(t)} = \{i \adjacent j : \Phi_{ij} \leq \alpha^{(t)} \}$
\State obtain separation sets $\bfS^{(t)} = \{\bfS(i, j) : \Phi_{ij} > \alpha^{(t)} \}$
\State{execute $\calG^{(t)} \gets \texttt{skel-to-cpdag($\calG^{(t)}, \bfS^{(t)}$)}$ (\autoref{alg:cpdag})}\label{algl:path_cpdag}
\State{obtain DAG extension $\tilde{\calG}^{(t)} \gets \texttt{pdag-to-dag($\calG^{(t)}$)}$ \citep{dor1992simple}} \label{algl:path_extend}
\If{$t = 1$}
\State $\Delta^{(t)} = 0$
\Else
\State compute score difference $\Delta^{(t)} = \phi(\tilde{\calG}^{(t)}, \calD) - \phi(\tilde{\calG}^{(t-1)}, \calD)$
\EndIf \label{algl:path_end_if}
\EndFor
\State select the best estimate $t^* = \argmax_{t \in \{1, \dots, \tau\}} \sum_{r = 1}^t \Delta^{(r)}$ \label{algl:score_pdag}
\State return $\calG^{(t^*)}$ and $\{\calG^{(t)} : t \in \{1, \dots, \tau\}\}$
\end{algorithmic}
\end{algorithm}

\begin{remark}\label{rmk:valid}

While \texttt{pdag-to-dag} is guaranteed to extend a PDAG $\calG$ to a DAG if any consistent extension exists, in the presence of finite-sample error, \autoref{alg:cpdag} may obtain a PDAG estimate $\calG$ for which no such extension exists. In such a case, we say that $\calG$ does not admit a consistent extension, and refer to it as an invalid CPDAG. Such PDAGs contain undirected edges that cannot be oriented without inducing cycles or constructing additional v-structures, and do not encode any probabilistic model of $P$. To account for these, we restrict the candidate graphs considered in line~\ref{algl:score_pdag} to valid CPDAGs. In the case that no valid CPDAG is obtained, we obtain semi-arbitrary DAG extensions $\tilde{\calG}^{(t)}$ by first applying the algorithm by \cite{dor1992simple} and randomly directing as many remaining undirected edges as possible without introducing any cycles, finally removing edges that cannot be oriented. The resulting DAGs are used to score the PDAGs, and the original PDAGs are returned in the output as these structures are nonetheless interpretable even as incomplete dependency structures. 

\end{remark}

The computational expense of executing \autoref{alg:path} can reasonably be expected to be insignificant compared to any constraint-based algorithm to which it is attached, supported by our empirical results in \autoref{sec:hgi_results}. Our exploitation of score decomposability reduces the order of score computations far below the worst case of $O(\tau p)$, which is already much more efficient than any favorable order of conditional independence tests such as polynomial with $p$. As for the $\tau$ executions of \texttt{skel-to-cpdag} and \texttt{pdag-to-dag}, \cite{chickering2002learning} found the computational cost of applying the edge orientation heuristics to be insignificant, as does \cite{madsen2017} for \texttt{skel-to-cpdag} in comparison to skeleton learning. As such, the computational cost for score-based selection from the solution path can be expected to be essentially inconsequential, an assertion further validated in our experiments. 

In the finite-sample setting, the thresholded estimate $\calG^{(t)}$ corresponding to $\alpha^{(t)} < \alpha^{(1)}$ typically does not correspond exactly to the estimate obtained by directly executing PC or pPC with significance level $\alpha^{(t)}$. Adjacencies that would survive to the later stages of the learning process with threshold $\alpha^{(1)}$ may be removed much earlier with a stricter threshold $\alpha^{(t)}$, resulting in fewer conditional independence evaluations and different conditioning sets considered for each node pair. Nonetheless, our empirical results demonstrate the potential of the solution path generated by thresholding, showing that PATH applied to pPC and PC is able to produce estimates of competitive quality to the best of those obtained by multiple executions with various $\alpha$. 

In the large-sample limit, the correctness of \eqref{eq:exist2} and \texttt{skel-to-cpdag} implies the asymptotic consistency of PC and pPC under certain conditions without any solution path (\autoref{lma:ppc_consistent}, \autoref{sec:proofs}). This property can be achieved with a consistent conditional independence test by controlling type~I error with $\alpha_n \to 0$ as $n \to \infty$, thus rendering the test Chernoff-consistent \citep[Definition 2.13]{cressie1989, shao2003}. As $\alpha_n$ is not practically informative due to its implicit dependence on $n$, \autoref{alg:path} contributes an element of accessibility to this asymptotic result in the form of the following theorem, a proof for which can be found in \autoref{sec:proofs}.

\begin{theorem}\label{thm:path}
Suppose the distribution $P$ is fixed and faithful to a DAG with CPDAG $\calG^*$, $\calD$ is data containing $n$ i.i.d. samples from $P$, and $\phi$ is a consistent score.
Let $\Phi_n=(\Phi_{n,ij})$ and $\bfS_n$ be the maximum p-values and corresponding separation sets recorded for any exhaustive investigation of \eqref{eq:exist2}, executed with a consistent test applied with threshold $\alpha_n$. Let $\hcalG_n^{(t^*)}(\alpha_n)$ be the selected estimate of \autoref{alg:path} with parameters $\tau = 1 + \sum_{i < j} \ind{\Phi_{n,ij} \leq \alpha_n}$ and $\alpha^{(\tau)} = 0$ applied to $\Phi_n$ and $\bfS_n$. Then there exists $a_n \to 0$ as $n \to \infty$ such that if $\alpha_n \geq a_n$ when $n$ is large, 
\begin{align}\label{eq:thm_path}
	\lim_{n \to \infty} \rmPr \left[ \hcalG_n^{(t^*)} (\alpha_n) = \calG^* \right] = 1. 
\end{align}
In particular, \eqref{eq:thm_path} holds if $\alpha_n$ is fixed to some constant $\alpha \in (0, 1)$ for all $n$.
\end{theorem}

To the best of our knowledge, there is currently no easy way to choose the optimal threshold $\alpha_n$ for PC (or pPC), and thus repeated executions are often needed for parameter tuning. The application of PATH pragmatically allows for a single execution of PC or pPC with fixed threshold $\alpha$ while returning estimates with both theoretical guarantees (\autoref{thm:path}) and empirical well-performance (\autoref{sec:ppc_results}).

\section{Consistent Hybrid Structure Learning}\label{sec:hybrid}

Despite the accessibility provided by PATH (\autoref{alg:path}), the asymptotic guarantees of constraint-based learning strategies do not necessarily translate to well-performance in practice. For this reason, %based on their experiments with which our simulation results are in agreement, 
\cite{tsamardinos2006max} preferred greedy search in the DAG space over sound and complete constraint-based orientation in their development of their algorithm, even though the former is lacking in comparable theoretical guarantees. 

In the interest of eliminating this compromise, we develop the hybrid greedy initialization (HGI) strategy to preserve the asymptotic guarantees of sound and complete constraint-based structure learning while improving on the empirical well-performance of the current standard hybrid framework. We motivate and develop HGI in this section, first reviewing relevant standard score-based and hybrid structure learning before describing the HGI strategy in detail.

\subsection{Score-based and Hybrid Structure Learning}\label{sec:score}

\cite{chickering2002learning} distills score-based Bayesian network learning into two general problems: the {evaluation} problem and the {identification} problem. In this section, we begin by introducing the relevant tenets of score-based structure learning under these categories. 

We have briefly interacted with the evaluation problem in our discussion of the BIC score in \autoref{sec:path}, which more broadly involves the development of scoring criteria to evaluate the goodness-of-fit of a Bayesian network to data. The existence of equivalence classes (\autoref{sec:equivalence}) motivates the design of metrics that evaluate all structures within an equivalence class as of equal quality, satisfying the \emph{score equivalence} property. The BIC score that we utilize satisfies this property, and is equivalent to the (negative) minimum description length (MDL) in \cite{rissanen1978}. Other scores that are score equivalent include log-likelihood, Akaike's information criterion (AIC), and Bayesian Dirichlet equivalence score (BDeu) \citep{akaike1974, buntine1991, heckerman1995}. Notwithstanding, prominent scores that are not score equivalent exist as well, such as the K2 score \citep{cooper1991} and, more recently, $\ell_1$-regularized likelihood scores \citep{fu2013, gu2019}. In their investigation, \cite{liu2012} found BIC to have favorable model selection performance relative to a number of other scores. 

The consistency of the BIC score guarantees that in the large-sample limit, $\calG^* = \argmax_{\calG} \phi(\calG, \calD)$ is in the equivalence class of the underlying DAG. BIC additionally retains the property of \emph{local consistency} \citep{chickering2002optimal}, meaning for any DAG $\calG$ and another DAG $\calG^\prime$ resulting from adding the edge $X_i \to X_j$ to $\calG$, the following two properties hold asymptotically:
\begin{align}
	\notindep{X_j}{X_i}{\Pa{j}} &\Rightarrow \phi(\bfX \mid \calG^\prime, \calD) > \phi(\bfX \mid \calG, \calD) \text{, and} \label{eq:local1} \\
	\indep{X_j}{X_i}{\Pa{j}} &\Rightarrow \phi(\bfX \mid \calG^\prime, \calD) < \phi(\bfX \mid \calG, \calD). \label{eq:local2}
\end{align}
We have discussed the BIC score as having desirable qualities for evaluating Bayesian network structures, being decomposable, score equivalent, consistent, locally consistent, and empirically well-performing. However, we are reminded of the problem of identification as finding the true global optimum $\calG^*$ is highly non-trivial. 

Relevant to our work is the general \emph{greedy search} algorithm which repeatedly moves from the current state to the neighboring state that maximally improves the optimization criterion (in our application, BIC) until no improvement can be thusly achieved \citep{russell2009}. That is, the algorithm is guaranteed to terminate in a locally optimal state, where locality is determined by the chosen definitions of a state and its neighborhood. The popular hill-climbing algorithm is a greedy search in the state space of DAGs, with neighboring states defined as DAGs obtainable by a single directed edge addition, deletion, or reversal applied to the current DAG \citep{heckerman1995, russell2009}. The greedy equivalence search (GES) is another variation of greedy search in which the state space is CPDAGs representing equivalence classes, with a forward-stepping edge addition phase followed by a edge deletion phase \citep{meek1997, chickering2002learning, chickering2002optimal}.

While widely regarded as efficient and well-performing, the locality of the hill-climbing search unavoidably risks the common problem of accepting locally optimal yet globally suboptimal solutions. \cite{gamez2011learning} showed that under certain conditions, hill-climbing returns a minimal independence map of $P$, but it does not guarantee a globally optimal result. Hill-climbing can be augmented to more thoroughly search the DAG space with one or both of \emph{tabu list} and \emph{random restarts}, governed by parameters $(t_0, t_1)$ and $(r_0, r_1)$ respectively. In what is known as the \emph{tabu search}, a solution is obtained through hill-climbing while a tabu list stores the last $t_1$ DAG structures visited during the search. Then, the hill-climbing procedure is continued for up to $t_0$ iterations while allowing for minimal score decreases, with a local neighborhood restricted by the tabu list to avoid previously visited structures. In hill-climbing with random restarts, the hill-climbing procedure is repeated $r_0$ times after the initial execution by perturbing the current solution with $r_1$ random local changes. In our work, we prefer augmenting hill-climbing with a tabu list rather than random restarts due to its generally superior efficiency and its reliable and deterministic well-performance. 

As mentioned in \autoref{sec:intro}, prominent hybrid structure learning algorithms are instantiations of what we call the generalized sparse candidate (GSC) framework, after the sparse candidate algorithm by \cite{friedman1999}, in which hill-climbing is executed from an empty graph restricted to a sparse set of candidate connections. That is, for a graph $\calG = (\bfV, \bfE)$ estimated using a constraint-based approach, define $\bfA = \{(i, j) : \text{$i$ and $j$ are connected in $\calG$} \}$ as the set of candidates: the unordered node pairs that have not been determined to be conditionally independent. Hill-climbing is then executed from an empty graph on $\bfV$, considering adding an edge $i \to j$ or $j \to i$ only if $(i, j) \in \bfA$. Max-min hill-climbing (MMHC) and hybrid HPC (H2PC) are two well-known examples, obtaining $\bfA$ according to sound skeleton estimation algorithms max-min parents and children (MMPC) and hybrid parents and children (HPC), respectively \citep{tsamardinos2006max, gasse2014}. The GSC strategy guarantees estimation of a valid DAG restricted to $\bfA$, but often accepts locally optimal solutions that are structurally inaccurate due to the connectivity of the DAG space induced by the hill-climbing neighborhood. As will be seen in \autoref{sec:hgi_results}, this problem persists even when the search space is well-restricted and a tabu list is utilized, leaving much to be desired.

\subsection{Hybrid Greedy Initialization}\label{sec:hgi}

We now develop our proposed HGI strategy to overcome the aforementioned difficulties for hybrid algorithms belonging to the GSC framework.  Our method is designed to retain the soundness and completeness of constraint-based structure learning, while empirically improving structural estimation accuracy and achieving higher-scoring structures as compared to those obtained by the GSC framework. The primary novel contribution is the introduction of a score-based ordering to the application of orientation heuristics to obtain a favorable initialization for hill-climbing. Given the skeleton output of a constraint-based algorithm, we sequentially add v-structures that most improve the score, scored with respect to directed edges. We then make greedy determinations for the remaining undirected edges according to efficient criteria from \texttt{pdag-to-dag} by \cite{dor1992simple}, assisted by Meek's rules R1-4. Finally, we execute hill-climbing initialized by the resulting DAG. From a score-based learning perspective, the formulation of HGI may be understood as a principled strategy for obtaining a good starting point for greedy search. In what follows, we further detail and discuss the HGI algorithm.

Recall that \texttt{pdag-to-dag} (introduced in \autoref{sec:path}) is guaranteed to obtain a consistent extension of a PDAG if one exists, and thus implicitly includes Meek's rules R1-4 when the given PDAG is a valid pattern that admits a consistent extension. Let $\calG_0$ and $\calG$ be identical copies of a PDAG to be oriented. The algorithm repeatedly searches for a node $j$ satisfying the following conditions in a PDAG $\calG_0$:
\begin{enumerate}[(a),topsep=0.5em,itemsep=0.25em]
	\item $j$ is a \emph{sink}: that is, $j$ has no edges directed outwards in $\calG_0$;\label{is_sink}
	\item For every vertex $k$ connected to $j$ by an undirected edge in $\calG_0$, $k$ is adjacent to all the other vertices which are adjacent to $j$ in $\calG_0$.\label{adj_check}
\end{enumerate}
If such a node $j$ can be found, all undirected edges adjacent to $j$ are oriented into $j$ in $\calG$ and $\calG_0$. Node $j$ is then a \emph{complete sink}, a node satisfying \ref{is_sink} with no undirected edges incident to it, and is removed from $\calG_0$ with all edges incident to it in order to uncover subsequent candidate nodes. This process is repeated until $\calG$ is fully oriented to a DAG, or until no such node $j$ can be found, in which case the initial PDAG does not admit a consistent extension. Briefly exposited, \ref{is_sink} ensures acyclicity by requiring that all directed paths induced by considered orientations terminate in sinks, and \ref{adj_check} ensures that considered orientations do not create new v-structures if applied.

\begin{figure}[h]
\centering
\begin{subfigure}[b]{0.3\textwidth}
\centering
\resizebox{0.9\linewidth}{!}{
\noindent\begin{tikzpicture}[main/.style = {draw, circle}, thick, node distance={16mm}] 
\node[main] (1) {$X_1$};   % satisfies conditions
\node[main] (2) [right of=1] {$X_2$};
\node[main] (3) [right of=2] {$X_3$}; 
\node[main] (4) [below of=1] {$X_4$};
\node[main] (5) [right of=4] {$X_5$}; 
\node[main] (6) [right of=5] {$X_6$};
\node[main] (7) [below of=4] {$X_7$};  % complete sink
\node[main] (8) [right of=7] {$X_8$};  % satisfies conditions
\draw (1) -- (4);  % to be oriented
\draw[->] (4) -- (5);
\draw[->] (2) -- (5);
\draw (3) -- (2);
\draw (3) -- (6);
\draw[->] (6) -- (8);  % to be deleted
\draw[->] (5) -- (8);  % to be deleted
\draw (5) -- (1);  % to be oriented
\draw (5) -- (7);  % to be oriented
\end{tikzpicture}
}
\caption{Pattern PDAG structure}
\label{fig:extend_eg_a}
\end{subfigure}
\begin{subfigure}[b]{0.3\textwidth}
\centering
\resizebox{0.9\linewidth}{!}{
\noindent\begin{tikzpicture}[main/.style = {draw, circle}, thick, node distance={16mm}] 
\node[main] (1) {$X_1$};
\node[main] (2) [right of=1] {$X_2$};
\node[main] (3) [right of=2] {$X_3$};
\node[main] (4) [below of=1] {$X_4$};  % satisfies conditions
\node[main] (5) [right of=4] {$X_5$};   % satisfies conditions
\node[main] (6) [right of=5] {$X_6$};
\node[main] (7) [below of=4] {$X_7$};
\node[main] (8) [right of=7] {$X_8$};
\draw[<-] (1) -- (4);  % deleted
\draw[->] (4) -- (5);  % deleted
\draw[->] (2) -- (5);  % deleted
\draw[->] (3) -- (2);  % to be oriented
\draw[->] (3) -- (6);  % deleted
\draw[->] (6) -- (8);  % deleted
\draw[->] (5) -- (8);  % deleted
\draw[->] (5) -- (1);  % deleted
\draw[->] (5) -- (7);  % deleted
\end{tikzpicture}
}
\caption{DAG extension}
\label{fig:extend_eg_b}
\end{subfigure}
\caption{Example of \texttt{pdag-to-dag}  \citep{dor1992simple} applied to a PDAG pattern structure.}
\label{fig:extend_eg}
\end{figure}
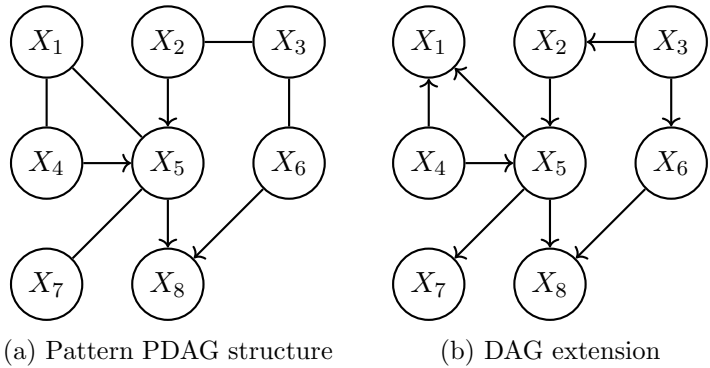

Consider an example of \texttt{pdag-to-dag} applied to a PDAG in \autoref{fig:extend_eg}. The algorithm proceeds as follows. Starting from the PDAG structure in \autoref{fig:extend_eg_a}, nodes $X_1$, $X_7$, and $X_8$ satisfy conditions \ref{is_sink} and \ref{adj_check}, with $X_8$ additionally a complete sink. Nodes $X_2$, $X_4$, and $X_6$ violate condition \ref{is_sink}, and $X_3$ and $X_5$ violate condition \ref{adj_check}. Since $X_1$, $X_7$, and $X_8$ are not adjacent to each other, they may be selected by the algorithm in an arbitrary order without affecting the particular outcome of the DAG extension, resulting in orientations $X_4 \to X_1$, $X_5 \to X_1$, and $X_5 \to X_7$. Once these nodes are removed from consideration, $X_5$ and $X_6$ are likewise removed as complete sinks, and the remaining undirected edge $X_2 \adjacent X_3$ may be oriented in either direction as both $X_2$ and $X_3$ satisfy \ref{is_sink} and \ref{adj_check}.

Since the implementation of \texttt{pdag-to-dag}, as proposed, does not straightforwardly lend itself to greedy application, we accomplish this by developing a decomposed version of \texttt{pdag-to-dag}. Let $\calG_0$ be a PDAG with only v-structures oriented, and let $\calG$ be a DAG consisting of only the directed edges in $\calG_0$. We prioritize checking for and removing all complete sinks from consideration by deleting all edges incident to such nodes in $\calG_0$, and we greedily consider orienting $i \to j$ in $\calG_0$ and $\calG$ if $i \adjacent j$ is an undirected edge in $\calG_0$ and conditions \ref{is_sink} and \ref{adj_check} are satisfied for node $j$. For example, in a single greedy step applied to \autoref{fig:extend_eg_a}, we would first remove $X_8$ from $\calG_0$ as a complete sink, resulting in nodes $X_1$, $X_6$, and $X_7$ satisfying \ref{is_sink} and \ref{adj_check}. We then greedily consider the individual edge orientations $X_4 \to X_1$, $X_5 \to X_1$, $X_5 \to X_7$, and $X_3 \to X_6$, applying the orientation that most improves the score computed with respect to the structure of $\calG$ (i.e., all edges that have determined orientations). This design essentially decomposes the node-centric operations in \texttt{pdag-to-dag} into single edge operations (e.g., $X_4 \to X_1$ and $X_5 \to X_1$ are considered as individual orientations instead of both being considered with node $X_1$), and its result is a DAG in the same equivalence class as the output of \texttt{pdag-to-dag} given a valid PDAG. In practice, as with the sequential v-structure application, the greedy ordering filters edges and selects between ambiguous orientations. In the case that undirected edges still exist in $\calG_0$ and no node satisfying \ref{is_sink} and \ref{adj_check} can be found, we likewise greedily consider transformations compelled by Meek's rules R1-4 applied to $\calG_0$. We detail the HGI strategy in \autoref{alg:hgi}. 

\begin{algorithm}[h]%[H]
\noindent
\begin{minipage}{\textwidth}
\renewcommand*\footnoterule{}  % switch off footnote line locally
\begin{savenotes}  % collect footnotes within algorithm environment
\caption{\texttt{HGI($\calG_0$, $\calD$, $\bfU$)} (sample version)}
\label{alg:hgi}
\begin{algorithmic}[1]
\Require{undirected graph $\calG_0$, data $\calD$, and v-structures $\bfU$}
\Ensure{DAG $\calG$}
\State initialize $\calG$ as the empty graph on $\bfV$
\Repeat\label{algl:hgs_3}
\State{orient $i \to k$ and $j \to k$ in $\calG_0$\footnote{Note that $\calG_0$ begins as an undirected graph and is oriented to a PDAG as orientations are applied.} and $\calG$ for $(i, k, j) \in \bfU$ that most improves $\phi(\calG, \calD)$ 
\Statex \hskip\algorithmicindent and does not induce any cycles or conflict with any v-structures in $\calG$}
\Until no such improvement possible
\Repeat
\State delete all edges incident to complete sinks in $\calG_0$
\State apply the first applicable of the following operations for $i \adjacent j$ in $\calG_0$:
	\begin{enumerate}[(i),leftmargin=3.5em]
		\item apply $i \to j$ to $\calG_0$ and $\calG$ satisfying \ref{is_sink} and \ref{adj_check} that most improves $\phi(\calG, \calD)$ \label{algl:hgs_i}
		\item delete $i \adjacent j$ from $\calG_0$ satisfying \ref{is_sink} and \ref{adj_check} that if applied, would most deteriorate $\phi(\calG, \calD)$ \label{algl:hgs_ii}
		\item apply $i \to j$ to $\calG_0$ and $\calG$ compelled by R1-4 that most improves $\phi(\calG, \calD)$ \label{algl:hgs_iii}
		\item delete $i \adjacent j$ from $\calG_0$ with $i \to j$ compelled by R1-4 such that if applied, would most deteriorate $\phi(\calG, \calD)$ \label{algl:hgs_iv}
	\end{enumerate} \label{algl:hgs_op}
\Until no such operation possible\label{algl:hgs_end_x}
\end{algorithmic}
\end{savenotes}
\end{minipage}
\end{algorithm}

Important to note is that in the population setting, v-structure detection and orientation is order-independent, and while the particular DAG obtained by (our decomposed) \texttt{pdag-to-dag} is order-dependent, it will always recover a DAG in the same equivalence class if successful (i.e., a consistent extension of the input PDAG exists). In such a case, whatever ordering imposed on both or either of the heuristics has no meaningful effect on the result. Furthermore, given a greedy criterion, a locally consistent score will asymptotically accept proposed additions of truly connected edges due to \eqref{eq:local1}, thus preserving guaranteed identification of the equivalence class of the underlying DAG. Indeed, in such a setting, a lenient score that prefers denser graphs is sufficient as only property \eqref{eq:local1} is required. 

In the finite-sample setting, incorrectly inferred conditional independence information can result in the determination of incomplete or extraneous and even conflicting v-structures, and could result in PDAGs that do not admit a consistent extension (\autoref{rmk:valid}). The outcome of a naive application of v-structures and (our decomposed) \texttt{pdag-to-dag} empirically varies in quality depending on the order by which the operations are applied due to conflicting operations and obstacles induced by the acyclicity constraint, providing the primary incentive for greedy decisions regarding proposed constraint-based orientations. From a constraint-based learning perspective, greedy forward stepping imposes a greedy ordering on the application of v-structures and other potentially conflicting or ambiguous edge orientations, while additionally providing an element of selectivity by disregarding operations that deteriorate the score. 

As already discussed, \autoref{alg:hgi} asymptotically preserves sound and complete orientation of the skeleton of the underlying DAG $\calG$ to a DAG in its equivalence class, straightforwardly evident from our discussion thus far.

\begin{lemma}\label{lma:hgi}

Suppose that probability distribution $P$ is fixed and faithful to a DAG $\calG^*$, $\calD_n$ is data containing $n$ i.i.d. samples from $P$, and $\phi$ is a score satisfying local consistency. Let $\hcalG_n$ be the output of \autoref{alg:hgi}. If $\calG_0$ is the skeleton of $\calG^*$ and $\bfU$ contains the v-structures of $\calG^*$, then $\hcalG_n$ is in the same equivalence class as $\calG^*$ with probability approaching one as $n\to\infty$.

\end{lemma}

Note that while we state \autoref{lma:hgi} assuming possession of all v-structures $\bfU$ that are present in the underlying DAG, these may be correctly obtained asymptotically depending on what information is available from the skeleton estimation method, which we discuss in \autoref{sec:orient}.

Indeed, neither operations \ref{algl:hgs_ii}-\ref{algl:hgs_iv} in line \ref{algl:hgs_op} nor any subsequent score-based search is necessary for \autoref{lma:hgi} to hold, but rather serve in a corrective capacity in the finite-sample setting. The process of completing and deleting sinks to uncover subsequent sinks in $\calG_0$ requires decisions for each undirected edge participating in a node satisfying \ref{is_sink} and \ref{adj_check} in order to continually progress in the algorithm. Operation \ref{algl:hgs_ii} discards each proposed edge addition $i \to j$ that deteriorates the score when no improvement according to \ref{algl:hgs_i} is possible in order for $j$ to be completed and removed. In the case that no node satisfying both \ref{is_sink} and \ref{adj_check} can be found, we apply the same greedy criterion to all edges compelled by Meek's rules R1-4 in operations \ref{algl:hgs_iii} and \ref{algl:hgs_iv}. These rules are not subject to a leaf-to-root construction and often help resume applications of \ref{algl:hgs_i} and \ref{algl:hgs_ii}, for example by deleting an undirected edge $i \adjacent j$ participating in an unshielded triple from $\calG_0$ so that $j$ can satisfy \ref{adj_check}. 

Note that in finite-sample applications, repeated application of \ref{algl:hgs_i}-\ref{algl:hgs_iv} does not guarantee orientation or deletion of all undirected edges in $\calG_0$ (e.g. consider an undirected square where no vertex satisfies \ref{is_sink} and \ref{adj_check} and no edge is compelled by R1-4), though we empirically find it to typically address most if not all edges. Furthermore, while the adjacency criterion \ref{adj_check} exists to prevent the creation of additional v-structures, it is still possible for new v-structures to be created by deletion. Consider an undirected triangle in $\calG_0$ where all three vertices $i$, $j$, and $k$ satisfy \ref{is_sink} and \ref{adj_check}. The greedy ordering may orient $i \to k$ and $j \to k$, remove node $k$ once it is a complete sink, and eventually delete $i \adjacent j$, leaving $i \to k \leftarrow j$ as a new v-structure in $\calG$. 

While essentially equivalent in the large sample setting, directly executing a greedy decomposed \texttt{pdag-to-dag} poses a number of pragmatic advantages over first greedily applying Meek's rules in the presence of finite-sample error. The sink criterion \ref{is_sink} effectively accomplishes acyclicity checks for each proposed edge orientation, which grow increasingly computationally burdensome for larger networks. It additionally induces a leaf-to-root construction with operations that minimally conflict with subsequent operations, with $i \to j$ only denying $i$ from satisfying \ref{is_sink} until $j$ is removed. This further strengthens the effect of the greedy ordering in minimizing ambiguity in the initial DAG construction process. Indeed, considering \autoref{thm:equivalence}, the order of greedy v-structure application is unambiguous given a score equivalent metric. In contrast, hill-climbing from an empty graph restricted to sparse candidates $\bfA$ begins with $O(\abs{\bfA}^2)$ ambiguous edge additions where, for any distinct node pair $(i, j) \in \bfA$, adding the edge $i \to j$ or $j \to i$ results in the same score improvement, again evident from \autoref{thm:equivalence}. Hill-climbing may encounter many such non-unique edge additions which are typically decided according to a node ordering that is often arbitrary, and their compounding effect can result in conflicts that, together with the acyclicity constraint, entrap hill-climbing in local solutions.

\begin{remark}\label{rmk:pef}

Relevant to our work is the partition, estimation, and fusion (PEF) framework by \cite{gu2020learning}, a hybrid strategy consisting of a final fusion step that is conceptually analogous to a non-greedy form of sparse candidate hill-climbing initialized with the directed edges of an estimated PDAG rather than an empty graph. The algorithm removes all undirected edges from a PDAG input and performs local edge additions, reversals, and deletions to the resulting DAG that improve the overall score by repeatedly iterating through the surviving node pairs in a semi-arbitrary order, simultaneously testing for conditional independence. While they empirically demonstrated this process to correct many of the errors in the estimated structure, we find that the order with which the edges are visited can result in varying degrees of improvement, and the testing strategy performs redundant conditional independence tests. Furthermore, naively initializing a score-based search with the PDAG output of a constraint-based algorithm may prove volatile given the sensitivity of \texttt{skel-to-cpdag} to erroneous conditional independence inferences as well as its order-dependence.
Lastly, even if initialized by the DAG consisting of the compelled edges of the underlying DAG and perfectly restricted to true connections, neither PEF nor hill-climbing with a consistent score guarantees asymptotic orientation to a DAG in the equivalence class of $\calG$. 

\end{remark}

\begin{algorithm}[h]
\noindent
\begin{minipage}{\textwidth}
\renewcommand*\footnoterule{}  % switch off footnote line locally
\begin{savenotes}  % collect footnotes within algorithm environment
\caption{\texttt{pHGS($\calD$, $\alpha$, $\tau$, $\alpha^{(\tau)}$)} (sample version)}
\label{alg:phgs}
\begin{algorithmic}[1]
\Require{data $\calD$, threshold $\alpha$, number of estimates $\tau$, minimum threshold $\alpha^{(\tau)}$}
\Ensure{DAG $\calG$}
\State execute $\texttt{pPC($\calD, \alpha$)}$ (the sample version of \autoref{alg:ppc}), omitting edge orientation in line \ref{algl:ppc_orient}, to obtain $\Phi$ and $\bfS$ as in \eqref{eq:sep}
\State execute $\calG^{(t^*)} \gets \texttt{PATH($\Phi, \bfS, \calD, \tau, \alpha^{(\tau)}$)}$ (\autoref{alg:path}) with the following modifications:
    \begin{itemize}[label={},leftmargin=1em]
    	\item line \ref{algl:path_cpdag}: detect v-structures $\bfU^{(t)}$ with $\bfS^{(t)}$ and execute $\calG^{(t)} \gets \texttt{HGI($\calG^{(t)}, \calD, \bfU^{(t)}$)}$ (\autoref{alg:hgi})
    	\item line \ref{algl:path_extend}: copy $\tilde{\calG}^{(t)} \gets \calG^{(t)}$
    \end{itemize} \label{algl:phgs_path}
\State execute hill-climbing\footnote{Or otherwise greedy search.} to obtain $\calG$, initialized with the selected estimate $\calG^{(t^*)}$ and restricted to $\bfA = \{ (i, j) : \Phi_{ij} \leq \alpha^{(1)} \}$ where $\alpha^{(1)} = \alpha$ is the maximum p-value (see \autoref{sec:path}) \label{algl:phgs_hc}
\end{algorithmic}
\end{savenotes}
\end{minipage}
\end{algorithm}

Finally, we detail in \autoref{alg:phgs} the partitioned hybrid greedy search (pHGS) algorithm, a composition of pPC, PATH, and HGI. The pHGS efficiently restricts the search space with the pPC algorithm (\autoref{alg:ppc}), obtaining $\Phi$ and $\bfS$ as in \eqref{eq:sep} for use in PATH. Instead of generating $\tau$ CPDAG estimates with \texttt{skel-to-cpdag}, PATH instead obtains $\tau$ DAG estimates by detecting v-structures $\bfU^{(t)}$ in each thresholded skeleton $\calG^{(t)}$ and executing HGI (\autoref{alg:hgi}). See \autoref{sec:orient} for details on v-structure detection. The highest-scoring of the $\tau$ estimates $\calG^{(t^*)}$ is selected to initialize hill-climbing (or an alternate score-based search algorithm) restricted to the active set $\bfA = \{ (i, j) : \Phi_{ij} \leq \alpha^{(1)} \}$. We choose the maximum threshold $\alpha^{(1)}$ (see \autoref{sec:path}) for the restriction instead of $\alpha^{(t^*)}$ corresponding to the highest-scoring estimate to reduce false negatives that excessively restrict the score-based exploration in the finite-sample setting. 

In the large-sample limit, under the same conditions and parameter specifications as \autoref{thm:path} and \autoref{lma:hgi}, the output of pHGS (\autoref{alg:phgs}) is a DAG that is Markov equivalent to the underlying DAG. Indeed, this result is already achieved by the modified PATH in line \ref{algl:phgs_path}, and in such a case the subsequent greedy search exists only to verify its optimality.

\section{Numerical Results}\label{sec:results}

We conducted extensive simulations considering various discrete Bayesian network configurations alongside a number of other popular structure learning algorithms to demonstrate the merits of pPC, PATH, HGI, and pHGS.

\subsection{Simulation Set-up}\label{sec:setup}

The performance of our methods were evaluated in comparison to several structure learning algorithms on numerous real Bayesian networks obtained from the Bayesian network repository compiled for the \texttt{R} package \texttt{bnlearn} \citep{scutari2010, scutari2017}. Most available discrete networks were considered, with the MUNIN networks represented by MUNIN1 and the LINK network omitted because certain minuscule marginal probabilities required much larger sample sizes to generate complete sets of valid variables. The following preprocessing procedures were applied to each network. For each random variable $X_i$, non-informative states $x_i$ with $\Pr(X_i = x_i) = 0$ were removed, and non-informative variables $X_i$ with $\abs{r_i} = 1$ were likewise removed. Furthermore, each variable $X_i$ was restricted to $\abs{r_i} \leq 8$, with the extraneous discrete states of excessively granular variables removed by randomly merging states. The conditional probability distributions imposed by merged states were averaged, weighted according to their marginal probabilities. 

In order to demonstrate the effectiveness of our methods for learning large discrete networks, we generated larger versions of each network with a house implementation of the tiling method proposed by \cite{tsamardinos2006tile}, modified to approximately preserve the average in-degree amongst non-minimal nodes. In particular, let $\calG = (\bfV, \bfE)$ be the structure consisting of $\kappa$ disconnected subgraphs to be connected by tiling. For a minimal node $k$ (that is, $k$ has no parents) with $d = d \coloneqq \max_{i \in \bfV} \abs{\Pa{i}}$ parents, instead of probabilistically choosing the number of added interconnecting edges, denoted by $e_k$, according to $e_k \sim {\rm Unif}\{0, d \coloneqq \max_{i \in \bfV} \abs{\Pa{i}} \}$, we let $\Pr (e_k = a) = \sum_{i \in \bfV} \ind{\abs{\Pa{i}} = a} / \abs{\bfV}$ for $a = 0, \dots, \min\{d,4\}$. Note that in this process we do not enforce any block structure on the tiled structures. 

The considered networks, along with selected descriptive characteristics, are presented in \autoref{tab:networks}, ordered by increasing complexity. The MIX network consists of the 14 networks from \autoref{tab:networks} with the least complexity, tiled in random order. For each network configuration, we generated $N = 100$ datasets with $n = 25000$ data samples each, for a total of 2000 datasets. The $p$ columns of each dataset were randomly permuted so as to obfuscate any information regarding the causal ordering. Note that while only one sample size was considered for all the networks of similar order in $p$, the networks vary significantly with respect to sparsity, structure, and complexity, thus representing a wide variety of conditions.

\begin{table}[h]
\begin{centering}
\begin{tabular}{rlrccccr}
\toprule
  \multicolumn{1}{c}{} & \multicolumn{1}{l}{Network} & \multicolumn{1}{c}{$\kappa$} & $p$                        & $\abs{\bfE}$                 & $\overline{\abs{\bfN^\mathcal{G}}}$ & $\max_i{\abs{\Pa{i}}}$     & \multicolumn{1}{r}{$\abs{\bm\Theta}$}\\
\midrule
1 & EARTHQUAKE & 200 & 1000 & 1103 & 2.206 & 2 & 2380\\
2 & CANCER & 200 & 1000 & 1123 & 2.246 & 2 & 2399\\
3 & ASIA & 125 & 1000 & 1243 & 2.486 & 2 & 2544\\
4 & SURVEY & 167 & 1002 & 1347 & 2.689 & 2 & 4444\\
5 & ANDES & 5 & 1115 & 2245 & 4.027 & 6 & 7082\\
6 & WIN95PTS & 14 & 1064 & 2184 & 4.105 & 7 & 9525\\
7 & CHILD & 50 & 1000 & 1309 & 2.618 & 2 & 11649\\
8 & ALARM & 28 & 1036 & 1689 & 3.261 & 4 & 16574\\
9 & MIX & 14 & 1011 & 1877 & 3.713 & 7 & 17319\\
10 & SACHS & 91 & 1001 & 1822 & 3.640 & 3 & 18746\\
11 & PIGS & 3 & 1323 & 2182 & 3.299 & 2 & 20102\\
12 & HEPAR2 & 15 & 1050 & 2077 & 3.956 & 6 & 23007\\
13 & INSURANCE & 38 & 1026 & 2120 & 4.133 & 3 & 40918\\
14 & HAILFINDER & 18 & 1008 & 1541 & 3.058 & 4 & 54322\\
15 & WATER & 39 & 1014 & 2303 & 4.542 & 5 & 72316\\
16 & MUNIN1 & 7 & 1064 & 1776 & 3.338 & 3 & 83208\\
17 & PATHFINDER & 10 & 1090 & 1968 & 3.611 & 5 & 96037\\
18 & DIABETES & 3 & 1239 & 2035 & 3.285 & 2 & 302008\\
19 & MILDEW & 29 & 1015 & 1913 & 3.769 & 3 & 345575\\
20 & BARLEY & 21 & 1008 & 2101 & 4.169 & 4 & 1771800\\
\bottomrule
\end{tabular}
\par\end{centering}
\caption{Simulated networks consisting of $\kappa$ connected sub-networks with $p$ nodes, $\abs{\bfE}$ edges, average number of neighbors $\overline{\abs{\bfN^\calG}}$, maximum in-degree $\max_i \abs{\Pa{i}}$, and $\abs{\bm \Theta}$ number of parameters.}
\label{tab:networks}
\end{table}

Algorithm implementations of competing algorithms MMPC, HPC, HITON, IAMB, MMHC, and H2PC, which we briefly introduce in their respective featuring sections, were obtained from the \texttt{R} package \texttt{bnlearn}, which is written in R with computationally intensive operations delegated to \texttt{C} \citep{scutari2010, scutari2017}. Our pPC, PATH, HGI, and pHGS implementations were built in \texttt{R} and \texttt{Rcpp} using tools from the \texttt{bnlearn} package, and the results for PC were obtained by executing pPC restricted to $\kappa = 1$, for fair comparison. 

We evaluate the quality of a graph estimate $\hcalG = (\bfV, \hat{\bfE})$ with respect to the underlying DAG $\calG = (\bfV, \bfE)$ by considering the Jaccard index of the CPDAG of $\hcalG$ in comparison to the CPDAG of $\calG$. The Jaccard index (JI) is a normalized measure of accuracy (higher is better), computed as JI~=~TP~/~($\abs{\bfE}+\abs{\hat{\bfE}}-\text{TP}$) where TP is the number of true positive edges: the number of edges in the CPDAG of $\hcalG$ that coincide exactly with the CPDAG of $\calG$ (both existence and orientation). 

We use the JI as our primary accuracy metric over the popular structural Hamming distance (SHD) as we find it to be consistent with SHD (higher JI almost always indicates lower SHD) and for its convenience as a normalized metric. The choice of evaluating the CPDAG estimates rather than DAG estimates is motivated foremost by the fact that given that our estimates are inferred from observational data, the orientation of reversible edges in DAGs provide no meaningful interpretation (see \autoref{sec:equivalence}). Additionally, the aforementioned metrics allow for evaluation of the quality of estimated PDAGs that do not admit a consistent extension (see \autoref{rmk:valid}). 

Regarding efficiency, execution time is confounded by factors such as hardware, software platform, and implementation quality. Even if the aforementioned variables are accounted for, performing all simulations on the same device cannot guarantee consistent performance over all simulations, and additionally severely constrains the feasible scope of study. We evaluate the estimation speed of structure learning algorithms by the number of statistical calls to conditional independence tests or local score differences, with fewer calls indicating greater efficiency. For pPC, we additionally include mutual information and entropy evaluations to account for the expense of clustering (see \autoref{sec:cluster}).

\subsection{pPC and PATH}\label{sec:ppc_results}

As the pPC algorithm can be considered an augmentation of the PC algorithm by imposing an ordering to the conditional independence tests by partitioning, we highlight its performance against the PC algorithm. We additionally apply the PATH augmentation to pPC and PC.

Note that our proposed HGI strategy motivates the design of high-performing constraint-based algorithms that not only efficiently restrict the search space, but also demonstrate potential for good score-based search initialization with HGI by producing structurally accurate estimates. As such, we further validate the performance of pPC and PATH against four other established constraint-based structure learning algorithms, all local discovery methods, modified with a symmetry correction for skeleton learning \citep{aliferis2010}. Max-min parents and children (MMPC) uses a heuristic that selects variables that maximize a minimum association measure before removing false positives by testing for conditional independence \citep{tsamardinos2003, tsamardinos2006max}. The fast version of the incremental association Markov blanket algorithm (Fast-IAMB; IAMB in this paper) is a two-phase algorithm that consists of a speculative stepwise forward variable selection phase designed to reduce the number of conditional independence tests as compared to single forward variable selection, followed by a backward variable pruning phase by testing for conditional independence \citep{tsamardinos2003iamb, yaramakala2005}. The semi-interleaved implementation of HITON\footnote{From the Greek word ``\emph{X\textgreek{iton}}", pronounced ``\emph{hee-t\'{o}n}", meaning ``cover", ``cloak", or ``blanket".} parents and children (SI-HITON-PC; HITON in this paper) iteratively selects variables based on maximum pairwise marginal association while attempting to eliminate selected variables by testing for conditional independence \citep{aliferis2003, aliferis2010}. Finally, hybrid parents and children (HPC) is comprised of several subroutines designed to efficiently control the false discovery rate while reducing false negatives by increasing the reliability of the tests \citep{gasse2014}. For each of these methods, following skeleton estimation, we orient edges by detecting and orienting v-structures according to \eqref{eq:vs2} and applying Meek's rules R1-4.

The maximum size of considered conditioning sets $m$ was chosen empirically for a balance between efficiency and well-performance: $m=3$ for pPC and PC, $m = 4$ for HPC, and $m = 5$ for the remaining methods. Note that HPC insignificantly varies in efficiency with $m$ and performs best with $m = 4$, and the remaining competing methods are more efficient but significantly less accurate with $m < 5$. We executed each algorithm on each network configuration for each of the following ten choices of significance level thresholds: 
\begin{align}\label{eq:alpha}
\alpha \in \mathcal{A} \coloneqq \{0.1, 0.05, 0.01, 0.005, 0.001, 0.0005, 0.0001, 0.00005, 0.00001, 0.000005 \}.
\end{align}
For each execution of pPC and PC, estimates for $\tau = 10$ thresholding values were automatically generated with PATH (\autoref{alg:path}) according to \eqref{eq:edge_diff}, restricted to a minimum value of $\alpha^{(\tau)} = 10^{-5}$.

\begin{table}[t]
\begin{centering}
\scalebox{1}{
\centering
\begin{tabular}{rrrrrrrrrr}
\toprule
\multicolumn{1}{c}{} & \multicolumn{6}{c}{\textbf{JI}} & \multicolumn{3}{c}{\textbf{Normalized Calls}} \\
\cmidrule(l{3pt}r{3pt}){2-7} \cmidrule(l{3pt}r{3pt}){8-10}
\multicolumn{1}{c}{} & \multicolumn{2}{c}{pPC} & \multicolumn{2}{c}{PC} & \multicolumn{1}{c}{pPC$^*$} & \multicolumn{1}{c}{PC$^*$} & \multicolumn{1}{c}{PC} & \multicolumn{1}{c}{pPC$^*$} & \multicolumn{1}{c}{PC$^*$} \\
\cmidrule(l{3pt}r{3pt}){2-3} \cmidrule(l{3pt}r{3pt}){4-5} \cmidrule(l{3pt}r{3pt}){6-6} \cmidrule(l{3pt}r{3pt}){7-7} \cmidrule(l{3pt}r{3pt}){8-8} \cmidrule(l{3pt}r{3pt}){9-9} \cmidrule(l{3pt}r{3pt}){10-10}
\multicolumn{1}{c}{} & \multicolumn{2}{c}{$\alpha = 0.1$} & \multicolumn{2}{c}{$\alpha = 0.1$} & \multicolumn{1}{c}{$\alpha \in \mathcal{A}$} & \multicolumn{1}{c}{$\alpha \in \mathcal{A}$} & \multicolumn{1}{c}{$\alpha = 0.1$} & \multicolumn{1}{c}{$\alpha \in \mathcal{A}$} & \multicolumn{1}{c}{$\alpha \in \mathcal{A}$} \\
\cmidrule(l{3pt}r{3pt}){2-3} \cmidrule(l{3pt}r{3pt}){4-5} \cmidrule(l{3pt}r{3pt}){6-6} \cmidrule(l{3pt}r{3pt}){7-7} \cmidrule(l{3pt}r{3pt}){8-8} \cmidrule(l{3pt}r{3pt}){9-9} \cmidrule(l{3pt}r{3pt}){10-10}
  & None & $\tau = 10$ & None & $\tau = 10$ & None & None & $\tau = 10$ & None & None\\
\midrule
1 & 0.449 & \textbf{0.726} & 0.390 & 0.694 & 0.712 & 0.707 & 3.256 & 8.537 & 11.543\\
2 & 0.249 & \textbf{0.333} & 0.259 & 0.331 & 0.292 & 0.292 & 7.842 & 7.258 & 15.863\\
3 & 0.220 & \textbf{0.292} & 0.211 & 0.290 & 0.258 & 0.259 & 1.434 & 4.629 & 5.199\\
4 & 0.326 & 0.359 & 0.309 & \textbf{0.361} & 0.359 & 0.359 & 2.811 & 7.934 & 10.474\\
5 & 0.504 & 0.560 & 0.488 & \textbf{0.561} & 0.519 & 0.508 & 4.334 & 7.681 & 12.905\\
6 & 0.407 & \textbf{0.414} & 0.385 & 0.405 & 0.412 & 0.400 & 2.490 & 8.377 & 11.049\\
7 & 0.568 & \textbf{0.583} & 0.538 & 0.564 & 0.573 & 0.552 & 1.585 & 8.478 & 9.480\\
8 & 0.475 & 0.478 & 0.469 & \textbf{0.480} & 0.477 & 0.474 & 1.811 & 8.193 & 10.013\\
9 & 0.518 & \textbf{0.548} & 0.505 & 0.537 & 0.520 & 0.508 & 1.296 & 7.515 & 8.166\\
10 & 0.621 & \textbf{0.633} & 0.572 & 0.595 & 0.621 & 0.574 & 1.894 & 8.249 & 10.089\\
11 & 0.852 & 0.853 & 0.867 & \textbf{0.871} & 0.852 & 0.867 & 1.561 & 8.938 & 10.747\\
12 & 0.206 & \textbf{0.235} & 0.199 & 0.226 & 0.207 & 0.201 & 2.235 & 7.729 & 10.013\\
13 & 0.407 & \textbf{0.426} & 0.404 & 0.421 & 0.411 & 0.407 & 1.898 & 8.175 & 10.095\\
14 & 0.339 & 0.359 & 0.341 & \textbf{0.361} & 0.342 & 0.347 & 1.465 & 8.624 & 9.697\\
15 & \textbf{0.307} & 0.296 & 0.291 & 0.262 & \textbf{0.307} & 0.291 & 1.874 & 8.537 & 10.396\\
16 & 0.086 & 0.085 & 0.087 & 0.087 & 0.088 & \textbf{0.090} & 2.009 & 8.620 & 13.546\\
17 & 0.053 & 0.053 & 0.053 & 0.052 & \textbf{0.054} & 0.053 & 1.266 & 8.844 & 9.940\\
18 & 0.219 & 0.219 & 0.262 & 0.262 & 0.239 & \textbf{0.263} & 1.765 & 8.744 & 11.738\\
19 & 0.323 & 0.297 & 0.312 & 0.274 & \textbf{0.326} & 0.313 & 1.664 & 8.707 & 10.410\\
20 & 0.158 & 0.157 & 0.149 & 0.149 & \textbf{0.168} & 0.160 & 1.684 & 8.289 & 9.840\\
\bottomrule
\end{tabular}
}
\par\end{centering}
\begin{centering}
\protect\caption{Accuracy (JI) and efficiency (Normalized Calls) comparison between pPC and PC, without and with PATH (indicated by None and $\tau = 10$, respectively), with number of statistical calls normalized by pPC-PATH$(\alpha = 0.1, \tau = 10)$. Rows correspond to the networks in \autoref{tab:networks}. Columns pPC$^*$ and PC$^*$ provide the highest JI and total statistical calls of executions for all ten $\alpha \in \mathcal{A}$. Best values are provided in boldface.}
\label{tab:ppc_pc}
\par\end{centering}
\end{table}

The comparison results for pPC and PATH are reported in \autoref{tab:ppc_pc}. We first compare pPC against PC in terms of computational efficiency and estimation accuracy. Unsurprisingly, pPC demonstrates the greatest computational benefit over PC for large $\alpha$, typically halving the number of conditional independence tests for $\alpha = 0.1$, as seen from the normalized calls of PC with $\alpha=0.1$ and PATH ($\tau=10$) in the table. Note that our partitioning strategy is solely responsible for this computational improvement as we do not consider any parallelization in our results, and by design pPC can, like PC, further benefit from parallel execution. The reduction suffers from diminishing returns with decreasing $\alpha$, with an average speed-up of about $20\%$ across the ten $\alpha$. Notwithstanding, we found pPC and PC, even without PATH, to generally prefer larger $\alpha$. In particular, for both pPC and PC, estimates with thresholds $\alpha = 0.1$ and $0.05$ produced the best estimates (highest JI) for over 240 of the 400 datasets, resulted in the highest average JI scores for 14 out of the 20 networks, and achieved the highest JI scores averaged across all datasets. As such, algorithm executions with large significance level thresholds are not unreasonable in practice, which coincides with the general strategy of PATH. We note that pPC appears to be typically slightly more accurate than PC, though the improvement is not substantial.

Additionally, we see from \autoref{tab:ppc_pc} that PATH applied to pPC and PC is able to obtain, from a single execution with $\alpha = 0.1$ and $\tau = 10$, estimates of similar and often superior quality compared to the best estimates without PATH (pPC$^*$ and PC$^*$) obtained from ten executions with the various $\alpha \in \mathcal{A}$. Important to note is that the solution path automatically selects an estimate based on a BIC selection strategy restricted to valid CPDAG estimates, if any (see \autoref{rmk:valid}), whereas for the multiple executions the maximum JI (as computed with respect to the CPDAG of the underlying DAG) for each dataset was chosen. The BIC selection strategy appears less effective for a couple of networks (e.g., 15 and 19), where on average the original estimates without PATH were more structurally accurate than those chosen from a solution path. One explanation for the worse performance could be the presence of invalid CPDAG estimates. PATH prefers valid estimates, and may prefer lower-scoring valid estimates over more structurally accurate invalid estimates. In the case that all estimates are invalid, the semi-arbitrary DAG extension process  can be volatile, resulting in structurally inaccurate estimates being selected. As anticipated in \autoref{sec:path}, the computational expense required to execute \autoref{alg:path} is practically negligible in comparison to skeleton estimation. The statistical calls for pPC and PC with PATH ($\alpha = 0.1$ and $\tau = 10$) include the scores evaluated for BIC selection from the generated solutions by PATH, and are practically indistinguishable from those without PATH, with the score evaluations typically consisting of less than $0.1\%$ of the statistical calls.

\begin{table}[t!]
\begin{centering}
\scalebox{1}{}
\centering
\begin{tabular}{rrrrrrrrrr}
\toprule
\multicolumn{1}{c}{} & \multicolumn{5}{c}{\textbf{JI}} & \multicolumn{4}{c}{\textbf{Normalized Calls}} \\
\cmidrule(l{3pt}r{3pt}){2-6} \cmidrule(l{3pt}r{3pt}){7-10}
  & pPC & MMPC & HPC & IAMB & HITON & MMPC & HPC & IAMB & HITON\\
\midrule
1 & 0.726 & 0.759 & \textbf{0.760} & 0.749 & 0.759 & 33.882 & 86.705 & 53.451 & 17.162\\
2 & 0.333 & 0.757 & 0.805 & \textbf{0.870} & 0.701 & 30.137 & 156.298 & 39.448 & 16.699\\
3 & 0.292 & \textbf{0.306} & 0.245 & 0.260 & 0.299 & 15.263 & 47.094 & 21.773 & 7.754\\
4 & 0.359 & 0.773 & 0.818 & \textbf{0.835} & 0.734 & 33.020 & 97.426 & 40.322 & 20.100\\
5 & 0.560 & 0.563 & \textbf{0.692} & 0.532 & 0.606 & 34.573 & 166.639 & 32.882 & 60.477\\
6 & \textbf{0.414} & 0.295 & 0.379 & 0.251 & 0.243 & 30.760 & 101.360 & 44.621 & 17.124\\
7 & 0.583 & 0.087 & \textbf{0.717} & 0.317 & 0.077 & 29.538 & 67.649 & 49.586 & 15.094\\
8 & \textbf{0.478} & 0.200 & 0.434 & 0.314 & 0.158 & 29.007 & 78.418 & 59.496 & 14.934\\
9 & 0.548 & 0.423 & \textbf{0.633} & 0.360 & 0.316 & 12.139 & 53.574 & 14.409 & 7.756\\
10 & 0.633 & 0.259 & \textbf{0.724} & 0.352 & 0.358 & 27.113 & 99.919 & 30.702 & 14.292\\
11 & 0.853 & 0.536 & \textbf{0.974} & 0.385 & 0.276 & 12.441 & 69.035 & 12.747 & 6.194\\
12 & 0.235 & 0.259 & \textbf{0.399} & 0.289 & 0.262 & 21.312 & 82.698 & 31.353 & 12.402\\
13 & \textbf{0.426} & 0.189 & 0.310 & 0.252 & 0.114 & 24.900 & 95.270 & 38.858 & 12.959\\
14 & \textbf{0.359} & 0.246 & 0.325 & 0.287 & 0.220 & 32.216 & 59.157 & 40.003 & 16.633\\
15 & \textbf{0.296} & 0.283 & 0.287 & 0.291 & 0.232 & 32.270 & 88.132 & 48.729 & 16.591\\
16 & \textbf{0.085} & 0.004 & 0.010 & 0.032 & 0.004 & 11.961 & 72.268 & 12.745 & 6.756\\
17 & 0.053 & 0.067 & \textbf{0.077} & 0.050 & 0.062 & 7.282 & 19.076 & 10.686 & 3.737\\
18 & 0.219 & 0.067 & \textbf{0.221} & 0.073 & 0.061 & 28.934 & 74.795 & 26.146 & 14.964\\
19 & \textbf{0.297} & 0.066 & 0.252 & 0.232 & 0.064 & 32.694 & 82.547 & 36.928 & 16.622\\
20 & 0.157 & 0.055 & \textbf{0.244} & 0.093 & 0.055 & 24.945 & 96.624 & 22.784 & 12.777\\
\bottomrule
\end{tabular}
\par\end{centering}
\begin{centering}
\protect\caption{Accuracy (JI) and efficiency (Normalized Calls) comparison amongst constraint-based methods, with total number of statistical calls normalized by pPC. Rows correspond to the networks in \autoref{tab:networks}. The results for pPC were obtained from a single execution with PATH and parameters $\alpha = 0.1$ and $\tau = 10$, whereas all other methods report the highest JI from and the total statistical calls for the executions across the ten $\alpha \in \mathcal{A}$. Best values are provided in boldface.}
\label{tab:ppc_all}
\par\end{centering}
\end{table}

In \autoref{tab:ppc_all}, we compare pPC with PATH against other constraint-based structure learning algorithms. We exclude PC as its comparison with pPC is thoroughly demonstrated in \autoref{tab:ppc_pc}. Again, competing methods report optimal results and total statistical calls for the executions across the ten significance levels $\alpha \in \mathcal{A}$. In terms of structural accuracy, the only algorithm that can compete against pPC is HPC, which outperforms pPC in twelve of the network configurations, sometimes by a substantial margin. However, when it comes to efficiency, there is no contest against the pPC algorithm, in general even if the number of calls were averaged across the ten executions instead of summed. Additionally, pPC most often produced valid CPDAG estimates, succeeding with $49.1\%$ of the datasets in contrast to from $9.8\%$ by HPC to up to $41.3\%$ by IAMB. 

In
\autoref{sec:supplement},
% the supplementary information, 
we provide detailed results for the comparisons of pPC against three constraint-based methods, PC, MMPC, and HPC, thus far discussed in \autoref{tab:ppc_pc} and \autoref{tab:ppc_all}. 
\autoref{fig:constraint_points}
% \sfref{supp-fig:constraint_points} 
% Supplementary Figure 1 
plots the accuracy and efficiency results obtained by these methods for each network, visualizing the variability amongst datasets by including the results for individual datasets. We also include detailed tables with additional metrics.

\subsection{HGI and pHGS}\label{sec:hgi_results}

Having discussed the theoretical merits of HGI in \autoref{sec:hgi}, we demonstrate the empirical performance of HGI applied to the GSC framework in this section. 
We refer to unrestricted hill-climbing as simply HC and perfectly restricted hill-climbing as GSC$^*$. In general, for a restriction of the search space with Alg, we refer to the GSC version as Alg-HC, and the version with HGI as Alg-HGI-HC. However, we refer to the versions of established algorithms MMHC and H2PC that are augmented with HGI as MMHC-HGI and H2PC-HGI. 

The hill-climbing phase of each algorithm was augmented with a tabu list to avoid $t_1 = 100$ previously visited DAG structures for $t_0 = 100$ suboptimal iterations (see \autoref{sec:score}). All score-based methods evaluated structures with the BIC score. We executed pHGS (\autoref{alg:phgs}) with significance level $\alpha = 0.05$, and generated and selected from $\tau = 10$ HGI estimates in PATH by thresholding to a minimum of $\alpha^{(\tau)} = 10^{-5}$.

\begin{table}[h]
\begin{minipage}{\textwidth}
\begin{centering}
\scalebox{1}{}
\centering
\begin{tabular}{rccccccccc}
\toprule
\multicolumn{1}{r}{Restrict} & \multicolumn{1}{c}{None} & \multicolumn{2}{c}{True Skeleton} & \multicolumn{2}{c}{pPC} & \multicolumn{2}{c}{MMPC} & \multicolumn{2}{c}{HPC} \\
\cmidrule(l{3pt}r{3pt}){1-1} \cmidrule(l{3pt}r{3pt}){2-2} \cmidrule(l{3pt}r{3pt}){3-4} \cmidrule(l{3pt}r{3pt}){5-6} \cmidrule(l{3pt}r{3pt}){7-8} \cmidrule(l{3pt}r{3pt}){9-10}
\multicolumn{1}{r}{Initial} & \multicolumn{1}{c}{EG} & \multicolumn{1}{c}{EG} & \multicolumn{1}{c}{HGI} & \multicolumn{1}{c}{EG} & \multicolumn{1}{c}{HGI} & \multicolumn{1}{c}{EG} & \multicolumn{1}{c}{HGI} & \multicolumn{1}{c}{EG} & \multicolumn{1}{c}{HGI} \\
\cmidrule(l{3pt}r{3pt}){1-1} \cmidrule(l{3pt}r{3pt}){2-2} \cmidrule(l{3pt}r{3pt}){3-3} \cmidrule(l{3pt}r{3pt}){4-4} \cmidrule(l{3pt}r{3pt}){5-5} \cmidrule(l{3pt}r{3pt}){6-6} \cmidrule(l{3pt}r{3pt}){7-7} \cmidrule(l{3pt}r{3pt}){8-8} \cmidrule(l{3pt}r{3pt}){9-9} \cmidrule(l{3pt}r{3pt}){10-10}
% \multicolumn{1}{c}{Method} & \multicolumn{1}{c}{HC} & \multicolumn{1}{c}{GSC} & \multicolumn{1}{c}{HGI-HC} & \multicolumn{1}{c}{GSC} & \multicolumn{1}{c}{pHGS} & \multicolumn{1}{c}{MMHC} & \multicolumn{1}{c}{HGI-HC} & \multicolumn{1}{c}{H2PC} & \multicolumn{1}{c}{HGI-HC} \\
\cmidrule(l{3pt}r{3pt}){1-1} \cmidrule(l{3pt}r{3pt}){2-2} \cmidrule(l{3pt}r{3pt}){3-3} \cmidrule(l{3pt}r{3pt}){4-4} \cmidrule(l{3pt}r{3pt}){5-5} \cmidrule(l{3pt}r{3pt}){6-6} \cmidrule(l{3pt}r{3pt}){7-7} \cmidrule(l{3pt}r{3pt}){8-8} \cmidrule(l{3pt}r{3pt}){9-9} \cmidrule(l{3pt}r{3pt}){10-10}
Alias  & HC & GSC$^*$ &  &  & pHGS & MMHC &  & H2PC & \\
\midrule
1 & 0.341 & 0.478 & 0.806 & 0.500 & \textbf{0.746} & 0.500 & 0.761 & 0.548 & 0.762\\
2 & 0.419 & 0.596 & 0.954 & 0.581 & \textbf{0.758} & 0.570 & 0.874 & 0.590 & 0.903\\
3 & \textbf{0.323} & 0.504 & 0.802 & 0.255 & 0.289 & 0.254 & 0.302 & 0.224 & 0.255\\
4 & 0.649 & 0.770 & 0.931 & 0.751 & \textbf{0.817} & 0.726 & 0.819 & 0.760 & 0.887\\
5 & 0.610 & 0.910 & 0.970 & 0.716 & \textbf{0.765} & 0.544 & 0.598 & 0.686 & 0.738\\
6 & 0.295 & 0.564 & 0.736 & 0.414 & \textbf{0.476} & 0.254 & 0.306 & 0.330 & 0.403\\
7 & 0.450 & 0.535 & 0.995 & 0.528 & \textbf{0.871} & 0.191 & 0.199 & 0.532 & 0.854\\
8 & 0.361 & 0.567 & 0.889 & 0.410 & \textbf{0.522} & 0.288 & 0.329 & 0.493 & 0.636\\
9 & 0.630 & 0.823 & 0.902 & 0.686 & \textbf{0.739} & 0.388 & 0.441 & 0.695 & 0.738\\
10 & 0.320 & 0.417 & 1.000 & 0.413 & \textbf{0.873} & 0.276 & 0.308 & 0.418 & 0.921\\
11 & 0.827 & 1.000 & 1.000 & 0.992 & \textbf{0.991} & 0.450 & 0.535 & 0.990 & 0.991\\
12 & 0.529 & 0.634 & 0.763 & 0.513 & 0.579 & 0.247 & 0.294 & \textbf{0.585} & 0.622\\
13 & 0.368 & 0.513 & 0.811 & 0.416 & \textbf{0.615} & 0.212 & 0.227 & 0.404 & 0.423\\
14 & \textbf{0.456} & 0.730 & 0.920 & 0.387 & 0.449 & 0.279 & 0.315 & 0.455 & 0.488\\
15 & 0.234 & 0.361 & 0.581 & 0.293 & \textbf{0.432} & 0.268 & 0.366 & 0.328 & 0.419\\
16 & \textbf{0.261} & 0.576 & 0.655 & 0.043 & 0.056 & 0.007 & 0.007 & 0.096 & 0.102\\
17 & \textbf{0.358} & 0.407 & 0.478 & 0.069 & 0.075 & 0.081 & 0.081 & 0.345 & 0.336\\
18 & 0.222 & 0.618 & 0.946 & 0.166 & 0.186 & 0.063 & 0.069 & \textbf{0.241} & 0.304\\
19 & 0.375 & 0.549 & 0.690 & 0.278 & 0.327 & 0.106 & 0.114 & \textbf{0.440} & 0.481\\
20 & 0.260 & 0.524 & 0.607 & 0.144 & 0.163 & 0.068 & 0.080 & \textbf{0.289} & 0.327\\
\bottomrule
\end{tabular}
\par\end{centering}
\begin{centering}
\protect\caption{Accuracy (JI) comparisons for CPDAGs estimated without and with HGI (indicated by EG for empty graph and HGI, respectively) given restrictions obtained by various skeleton methods. Unrestricted hill-climbing is provided under HC. HGI under pPC corresponds to a single execution of pHGS with $\alpha = 0.05$, $\tau = 10$, and $\alpha^{(\tau)} = 10^{-5}$. All methods other than HC and pHGS report the highest JI for the executions across the ten $\alpha \in \mathcal{A}$. Best values amongst pHGS, HC, MMHC, and H2PC are provided in boldface.}
\label{tab:hgs_ji}
\par\end{centering}
\end{minipage}
\end{table}

The accuracy results for HGI and pHGS are summarized in \autoref{tab:hgs_ji}. In the first column, the performance of (unrestricted) HC leaves much to be desired with its generally lackluster structural accuracy in comparison with the hybrid methods. Exceptions exist, as anticipated by \autoref{tab:ppc_all} in which constraint-based methods struggle to produce good estimates for some higher complexity networks, often inferring excessive false negatives that erroneously reduce the search space. In such cases, hybrid approaches are limited by their constraint-based component and thus perform worse than HC. 

We first demonstrate the improvement of initialization with HGI compared to the empty graph (EG) for different skeleton restriction methods. A perfect restriction to the true skeleton represents the most optimistic scenario for GSC and HGI, wherein all true positives are accessible and no false positives are considered. HGI additionally enjoys consideration of all true v-structures when detecting v-structures amongst unshielded triples according to \eqref{eq:vs2}. Unsurprisingly, perfectly restricted GSC uniformly improves on the performance of unrestricted HC. The addition of HGI achieves further improvements to structural accuracy of typically 28\% and up to 139\% (\autoref{tab:hgs_ji}, True Skeleton). This same trend persists when comparing GSC without and with HGI for empirical skeleton estimation methods pPC, MMPC, and HPC, demonstrating the effectiveness of HGI. For the various skeleton estimation methods, the addition of HGI achieves estimates that are typically 14\% and up to 118\% more structurally accurate.

We now compare pHGS with established algorithms HC, MMHC, and H2PC, the best values amongst which are indicated in boldface in \autoref{tab:hgs_ji}. In most of the network configurations, a single execution of pHGS learns estimates of higher quality than the best of the ten executions with $\alpha \in \mathcal{A}$ of MMHC and H2PC. Note that for the GSC framework, the goal of parameter tuning for $\alpha$ is to obtain a balance between true positives and true negatives. MMHC does not outperform pHGS in any meaningful capacity, and H2PC only substantially outperforms pHGS for higher complexity networks due to a mechanism in HPC to reduce false negative edges \citep{gasse2014}. While the addition of HGI dramatically improves the general accuracy of MMHC, only H2PC-HGI performs competitively with pHGS, reflective of the results in \autoref{tab:ppc_all} where HPC rivaled pPC with respect to structural accuracy. However, as we will see from our discussion of \autoref{tab:hgi_calls}, the speed comparisons in \autoref{tab:ppc_all} generally hold for these hybrid variants as well, with pHGS on average nearly an order of magnitude more efficient than H2PC and around 2.5 times more efficient than MMHC per execution, with or without HGI. HC only outperforms the hybrid methods for a few structures in which the latter overly restrict the search space. Overall, we find pHGS to be most well-performing method, followed by H2PC, MMHC, and HC. We provide detailed results for these methods in 
\autoref{sec:supplement},
% the supplementary information, 
with boxplots 
(\autoref{fig:established_boxplots})
% (\sfref{supp-fig:established_boxplots}) 
% (Supplementary Figure 2) 
visualizing the accuracy comparisons across datasets for each network as well as tables with additional metrics.

\begin{table}[h]
\begin{centering}
\scalebox{0.94}{
\centering
\begin{tabular}{llllll}
\toprule
\multicolumn{1}{c}{} & \multicolumn{1}{c}{GSC} & \multicolumn{4}{c}{HGI-HC} \\
\cmidrule(l{3pt}r{3pt}){2-2} \cmidrule(l{3pt}r{3pt}){3-6}
  & HC & Detect ($\bfU$) & HGI & HC & Total\\
\midrule
pPC & 1 (0.04, 1.87) & 0 (0, 0) & 1.08 (0.05, 3.51) & 0.48 (0.03, 1.11) & 1.59 (0.11, 4.25)\\
MMPC & 0.13 (0.01, 0.28) & 0.9 (0.02, 51.6) & 0.02 (0, 0.06) & 0.06 (0.01, 0.16) & 1.01 (0.03, 51.7)\\
HPC & 0.22 (0.02, 0.41) & 0.92 (0, 32.6) & 0.04 (0.01, 0.08) & 0.12 (0.02, 0.21) & 1.11 (0.04, 32.8)\\
\bottomrule
\end{tabular}
}
\par\end{centering}
\begin{centering}
\protect\caption{Median and $95\%$ precision intervals of percent additional statistical calls by GSC and HGI-HC with respect to skeleton learning. Each point represents one algorithm execution for each network. HGI with pPC represents pHGS, which includes PATH executed with $\tau = 10$ and $\alpha = 0.05$, whereas MMPC and HPC include the results from each individual execution from the ten $\alpha \in \mathcal{A}$.}
\label{tab:hgi_calls}
\par\end{centering}
\end{table}

As evidenced in \autoref{tab:hgi_calls}, while HGI-HC typically comes at greater computational cost to GSC alone, the expense of either is largely negligible in comparison to that of skeleton learning, generally (and often significantly) fewer than an additional $2\%$ statistical calls. Rare exceptions exist, in particular extreme cases where MMPC or HPC required a significant number of additional tests to detect v-structures. Here, pPC has a clear computational advantage, having the ability to detect v-structures using separation sets accrued throughout skeleton learning (see \eqref{eq:vs1}, resulting in fewer than $4.5\%$ additional calls for every dataset to execute HGI $\tau = 10$ times in PATH and perform hill-climbing from the chosen initial DAG. Other algorithms must conduct additional conditional independence tests to detect v-structures via \eqref{eq:vs2}, which can quickly add up if the learned skeleton structure has a significant number of unshielded triples $i \adjacent k \adjacent j$, or if either or both of $\abs{\bfN_i^\calG}$ and $\abs{\bfN_j^\calG}$ are large. On the topic of efficiency, unrestricted HC typically requires three to five times the number of statistical calls to execute as compared to pHGS, providing further validation for the hybrid approach.

In general, we find the initial DAG obtained by HGI (\autoref{alg:hgi}) through the greedy application of v-structures and greedy decomposed \texttt{pdag-to-dag} to be typically superior in structural accuracy compared to the direct application of \texttt{skel-to-cpdag} (\autoref{alg:cpdag}), the standard edge orientation strategy of constraint-based algorithms. HGI exhibits the greatest median improvement of $16.2\%$ over \texttt{skel-to-cpdag} when applied to pPC, followed by $9.5\%$ with HPC, $7.3\%$ restricted to the skeleton of the underlying DAG (True), and $1.9\%$ with MMPC. In general, pPC detects the most v-structures as its detection criterion \eqref{eq:vs1} may be considered less strict than \eqref{eq:vs2} used by True, MMPC, and HPC. Consequently, pPC generally detects a significant number of false positive v-structures, thus benefiting most significantly from the greedy v-structure determinations. The poor skeleton estimation performance of MMPC is likely responsible for its lackluster improvement, with its estimated skeletons generally containing the fewest unshielded triples corresponding to true v-structures in the underlying DAGs in comparison to pPC and HPC.

\section{Discussion}\label{sec:discussion}

In this paper, we proposed three independent yet compatible contributions to the general well-performance of discrete Bayesian network structure learning, culminated in the form of the pHGS algorithm. 

First, the pPC algorithm improves on the empirical efficiency of the PC algorithm while retaining its soundness and completeness as well as its empirical performance. Though it is difficult to quantify the expected computational reduction, our simulation results in \autoref{tab:ppc_pc} indicate that for the empirically preferred significance level threshold, pPC typically requires half the number of conditional independence tests per execution compared to PC. This speed-up is enjoyed at no compromise to structural accuracy, with pPC performing comparably with and often even outperforming PC. 

Second, the PATH algorithm effectively accomplishes the task of parameter tuning for certain constraint-based structure learning algorithms such as pPC and PC, theoretically preserving consistency in the classical setting and empirically achieving highly competitive structural accuracy at negligible computational expense. In the current landscape, the asymptotic result for sound and complete constraint-based structure learning asserts the existence of some uninformative sequence of significance levels $\alpha_n \to 0$ as $n \to \infty$ that recovers the underlying equivalence class in the large-sample limit. We prove that appropriately applied to pPC or PC executed with any fixed threshold $\alpha \in (0, 1)$, PATH asymptotically includes and correctly selects the underlying CPDAG in its generated solution path. We demonstrate an analogous result in the empirical setting, wherein pPC with PATH returns estimates of competitive quality to that of optimistic parameter tuning, achieving significant computational reductions compared to the current standard.

Third, the HGI algorithm provides a principled strategy for initializing score-based search in hybrid methods that asymptotically preserves soundness and completeness of constraint-based structure learning, elevating the GSC framework to consistency in the classical setting while significantly improving its empirical performance. While popular hybrid algorithms MMHC and H2PC forego asymptotic consistency for empirical performance, our HGI strategy makes no such compromise. When applied to GSC with various skeleton estimation strategies (including MMHC and H2PC), HGI significantly improves the estimation accuracy with typically negligible additional computational expense. Notably, a more recent development in hybrid structure learning is adaptively restricted greedy equivalence search (ARGES), which adaptively relaxes the restricted search space in order to ensure a search path in the space of equivalence classes (CPDAGs) to the optimal solution \citep{nandy2018}. Though in this paper we take primary interest in improving upon the GSC framework which operates in the space of DAGs, we present preliminary simulation results in 
\autoref{sec:supplement}
% Section 2 of the supplementary information 
that indicate significant potential for empirical improvement to ARGES through the initialization provided by HGI.

Altogether, we combined pPC, PATH, and HGI to create the pHGS algorithm, which enjoys the skeleton estimation efficiency of pPC, the parameter tuning by PATH, and and the empirical well-performance of GSC with HGI. In comparison to MMHC and H2PC, pHGS learns more accurate DAGs in nearly every considered underlying network configuration with a fraction of the computational cost.

While we have empirically demonstrated a significant reduction to the number of conditional independence tests executed by pPC in comparison to PC, we have not established any concrete theoretical complexity results. Indeed, the extent of computational reduction inevitably depends on the quality of the partition with respect to the underlying structure. As such, it is of interest to determine under what conditions pPC is guaranteed to perform fewer tests than PC, and to quantify the reduction. Such an investigation is crucial to establish structure learning consistency of pPC in the sparse high-dimensional setting with multivariate Gaussian distributions. The high-dimensional Gaussian consistency of PC proved in \cite{kalisch2007} relies on the number of conditional independence tests performed, determined by the maximum size of conditioning sets reached by PC with no errors. The same result holds for pPC under the same assumptions if the number of tests investigated by pPC is not greater than that of PC. Note that empirically, we find pPC to always perform fewer tests than PC.

Additionally, as formulated, pPC is limited to obtaining $\kappa \leq 20$ clusters as per a loose suggestion by \cite{hartigan1981}. Note that by our design, pPC is not limited to parallel processing utility from at most 20 processors, having comparable capacity for parallel execution as PC. Nonetheless, a future direction for further improvement would be the development of an unsupervised criterion to more flexibly determine the target number of clusters that optimizes the efficiency of pPC.

\bibliographystyle{plainnat}
\bibliography{references}

\newpage
\appendix

\section{Proofs}\label{sec:proofs}

In this appendix, we prove \autoref{thm:ppc} and \autoref{thm:path}.

\subsection{Proof of \autoref{thm:ppc}}\label{sec:ppc_proof}

The proof of \autoref{thm:ppc} regarding the soundness and completeness of the pPC algorithm (\autoref{alg:ppc}) follows from the thorough investigation of the criterion \eqref{eq:exist2} in the pPC algorithm.

\begin{proof}[Proof of \autoref{thm:ppc}]

Let $\bfc$ be any clustering labels obtained in the clustering step. Let $\hcalG_1$ denote the estimated structure after estimating edges within clusters at line \ref{algl:between_fix1}, and $\tdcalG$ denote the final estimated skeleton at line \ref{algl:ppc_skeleton}.

We first show that every truly adjacent node pair (that is, adjacent in ${\calG^*}$) is also adjacent in the estimated structure at every stage of \autoref{alg:ppc} following line \ref{algl:end_screen1}. An implication of the Markov condition (\autoref{sec:background}) is that no truly adjacent pair of nodes $i, j \in \bfV$ are independent conditioned on any set of variables not containing $X_i$ or $X_j$. Thus, every truly adjacent pair of nodes belonging to the same cluster are never disconnected after being connected in the first execution of \autoref{alg:ppc} in line \ref{algl:pc_within}:
\begin{align}\label{eq:super_within}
    \bfN_i^{\hcalG_1} \supseteq \{ X_k \in \bfN_i^{\calG^*} : c_k = c_i \} \text{ for all } i \in \bfV.
\end{align}
Similarly, every truly adjacent pair of nodes belonging to different clusters will necessarily be connected in the first screening of edges between clusters, and are never disconnected. Thus, $\tdcalG$ is equivalent to or a supergraph of the skeleton of ${\calG^*}$, with
\begin{align}\label{eq:super_all}
    \bfN_i^{\tdcalG} \supseteq \bfN_i^{\calG^*} \text{ for all } i \in \bfV.
\end{align}

Next, we show that no truly nonadjacent node pair (that is, not adjacent in ${\calG^*}$) will be connected in $\tdcalG$. Restating \eqref{eq:exist2}, for every truly nonadjacent distinct node pair $i, j \in \bfV$ there exists at least one conditioning set $\bfX_\bfk \subseteq \bfN_i^{\calG^*} \setminus X_j$ or $\bfX_\bfk \subseteq \bfN_j^{\calG^*} \setminus X_i$ such that $\indep{X_i}{X_j}{\bfX_\bfk}$. Thus, it is sufficient to show that for every truly nonadjacent node pair, $\indep{X_i}{X_j}{\bfX_\bfk}$ for one such separation set $\bfS(i,j) \coloneqq \bfk$ is evaluated by \autoref{alg:ppc}. Consider the following for any truly nonadjacent node pair $i, j \in \bfV$.

Suppose $c_i = c_j$. If there exists a conditioning set $\bfX_\bfk$ such that $\indep{X_i}{X_j}{\bfX_\bfk}$ and $c_k = c_j$ for all $k \in \bfk$, then since \eqref{eq:super_within}, $\indep{X_i}{X_j}{\bfX_\bfk}$ is evaluated by the first modified execution of PC in line \ref{algl:pc_within}. Otherwise, if $\exists k \in \bfk$ such that $c_k \neq c_j$, then for any and every conditioning set $\bfX_\bfk$ such that $\indep{X_i}{X_j}{\bfX_\bfk}$, $\bfX_\bfk$ satisfies criteria (a) of \eqref{eq:criteria}. In such a case, since \eqref{eq:super_all}, $\indep{X_i}{X_j}{\bfX_\bfk}$ for one such conditioning set $\bfX_\bfk$ is evaluated by the second modified execution of PC in line \ref{algl:ppc_skeleton}.

Suppose $c_i \neq c_j$. Ideally, $i$ and $j$ will be separated during the screening of edges between clusters in lines \ref{algl:between_fix1}-\ref{algl:end_screen2}. If they are not, then for any and every conditioning set $\bfX_\bfk$ such that $\indep{X_i}{X_j}{\bfX_\bfk}$, $\bfX_\bfk$ satisfies criteria (b) of \eqref{eq:criteria}. Since \eqref{eq:super_all}, $\indep{X_i}{X_j}{\bfX_\bfk}$ for one such $\bfX_\bfk$ is evaluated by the second modified execution of PC in line \ref{algl:ppc_skeleton}. 

We have shown that \autoref{alg:ppc} correctly estimates the skeleton of ${\calG^*}$ and, in particular, for every truly nonadjacent $i, j \in \bfV$, $\indep{X_i}{X_j}{\bfX_\bfk}$ for some conditioning set $\bfX_\bfk$ has been evaluated, with separation set $\bfS(i,j) \coloneqq \bfk$ stored as specified in \autoref{alg:ppc}. Since these separation sets are guaranteed to be accurate with conditional independence oracles, orientation of $\tdcalG$ to the CPDAG of ${\calG^*}$ in line \ref{algl:ppc_orient} follows directly from the correctness of \texttt{skel-to-cpdag} (\autoref{alg:cpdag}) \citep{spirtes1991, meek1995, kalisch2007}. 

\end{proof}

\subsection{Proof of \autoref{thm:path}}\label{sec:path_proof}

To prove \autoref{thm:path}, we begin by stating and proving a well-known consistency result for constraint-based structure learning in the classical setting. 

\begin{lemma}\label{lma:ppc_consistent}
Suppose the distribution $P$ is fixed and faithful to a DAG with CPDAG ${\calG^*}$ and $\calD$ is data containing $n$ i.i.d. samples from $P$. Let $\hcalG_n (\alpha_n) = (\bfV, \hat{\bfE}_n)$ be the graph output of any exhaustive investigation of \eqref{eq:exist2} followed by \autoref{alg:cpdag} executed with a consistent test applied with threshold $\alpha_n$. Then there exists $\alpha_n \to 0$ $(n \to \infty)$ such that
\begin{align}\label{eq:lma_ppc_consistent}
	\lim_{n \to \infty} \rmPr \left[ \hcalG_n (\alpha_n) = {\calG^*} \right] = 1. 
\end{align}
\end{lemma}

\begin{proof}[Proof of \autoref{lma:ppc_consistent}]
Define $p_{n;i,j|\bfk}$ and $p_{n;i,j|\bfk}^*$ be the p-values for testing the independence between $i$ and $j$ conditioned on $\bfk$ with $n$ data samples when $i$ and $j$ are and are not d-separated by $\bfk$ in ${\calG^*}$, respectively. For example,
\begin{align*}
	p_{n;i,j\mid\bfk} &= \rmPr(\chi^2_f > G^2_{n;ij | \bfk} \mid \dsep{X_i}{X_j}{\bfX_{\bfk}}) \text{, and } \\ 
	p_{n;i,j\mid\bfk}^* &= \rmPr(\chi^2_f > G^2_{n;ij | \bfk} \mid \notdsep{X_i}{X_j}{\bfX_{\bfk}})
\end{align*}
for the $G^2$ test of independence. Given faithfulness, d-separation in ${\calG^*}$ corresponds one-to-one with conditional independence in $P$. In the case that conditional independence holds (i.e. the null hypothesis is true), $p_{n;i,j|\bfk} \sim \text{Unif}(0, 1)$, so $\rmPr(p_{n;i,j|\bfk} \leq \alpha_n) = \alpha_n \to 0$ $(n \to \infty)$. Whereas whenever conditional independence does not hold, the consistency of the conditional independence test implies that $p^*_{n;i,j|\bfk} = o(1)$. As these statements hold for every $i, j \in \bfV$ and $\bfk \subseteq \bfV \setminus \{i, j\}$ with fixed $p = \abs{\bfV}$, what follows is the existence of $\alpha_n \to 0$ such that the probability of making any erroneous conditional independence inference in the investigation of \eqref{eq:exist2} decays to zero when $n\to\infty$. That is,
\begin{align}\label{eq:ppc_errors}
	\lim_{n \to \infty} \rmPr \left[ \bigcup_{i,j,\bfk} \left( p_{n;i,j|\bfk} \leq \alpha_n \right) \cup \left( p^*_{n;i,j|\bfk} > \alpha_n \right)  \right] 
		&= 0.
\end{align}
Given perfect conditional independence inferences asymptotically, the consistency follows straightforwardly from the correctness of \eqref{eq:exist2} for skeleton identification and the soundness and completeness of \autoref{alg:cpdag} for CPDAG orientation. 

\end{proof}

The proof of \autoref{thm:path} then proceeds from \autoref{lma:ppc_consistent} and its proof as follows.

\begin{proof}[Proof of \autoref{thm:path}]

\eqref{eq:ppc_errors} indicates the existence of some $a_n \to 0$ $(n \to \infty)$ such that
\begin{align}\label{eq:consistent}
	\lim_{n \to \infty} \rmPr \left[ \sup_{i, j, \bfk} p^*_{n;i,j|\bfk}
			\leq a_n <  \inf_{i, j, \bfk} p_{n;i,j|\bfk} \right] = 1.
\end{align}
That is, in the large-sample limit, $a_n$ perfectly discriminates true positives from true negatives. 
For conditional independence inferred with threshold $\alpha_n \geq a_n \geq \sup_{i, j, \bfk} p^*_{n;i,j|\bfk}$, only false positive edges are possible. This ensures that for every truly nonadjacent distinct node pair $i, j \in \bfV$ (that is, $i \adjacent j \not\in \bfE^*$ where $\bfE^*$ is the edge set of the skeleton of $\calG^*$), at least one conditioning set that separates $i$ and $j$ is considered in the investigation of \eqref{eq:exist2} and stored in \eqref{eq:sep} since \eqref{eq:consistent} holds asymptotically.
\begin{align}\label{eq:p_values}
	\max_{i \adjacent j \in \bfE^*} \Phi_{n;ij} &\leq \sup_{i, j, \bfk} p^*_{n;i,j|\bfk} 
		\leq a_n < \inf_{i, j, \bfk} p_{n;i,j|\bfk} \leq \min_{i \adjacent j \not\in \bfE^*} \Phi_{n;ij}.
\end{align}
Inspecting \eqref{eq:p_values} confirms the existence of a threshold value $\hat{a}_n$, in particular $\hat{a}_n = \max_{i \adjacent j \in \bfE^*} \Phi_{n;i,j}$, that perfectly discriminates true positives from true negatives. Consequently, $\bfE^* = \{i \adjacent j : \Phi_{n;ij} \leq \hat{a}_n \}$, and $\{\bfS_n(i, j) : \Phi_{n;ij} > \hat{a}_n \}$ obtains correct separation sets for complete and sound orientation to ${\calG^*}$ according to \autoref{alg:cpdag}. 

Having verified the existence of a threshold $\hat{a}_n$ that recovers ${\calG^*}$ by thresholding $\Phi_n$ and $\bfS_n$, it remains to show that $\hat{a}_n$ is recovered and selected by \autoref{alg:path}. Parameter settings $\tau = 1 + \sum_{i < j} \ind{\Phi_{n;ij} \leq \alpha_n} = 1 + \abs{\hat{\bfE}_n^{(1)}}$ and $\alpha^{(\tau)} = 0$ regulate the generation of $\alpha^{(t)}$ such that the skeletons of sequential estimates differ in sparsity by exactly one edge. The resulting values $\mathcal{A} \coloneqq \{\alpha^{(1)}, \ldots, \alpha^{\tau}\}$ correspond to the order statistics of the upper triangular elements of $\Phi_n$ decreasing from $\alpha_n^{(1)} = \alpha_n \geq a_n \geq \hat{a}_n$ to $\alpha_n^{(\tau)} = 0$, ensuring that $\hat{a}_n = \max_{i \adjacent j \in \bfE^*} \Phi_{n;i,j} \in \mathcal{A}$ and so ${\calG^*} \in \{ \calG^{(t)} : t \in \{1, 2, \dots, \tau \} \}$. The consistency of $\phi$ ensures that in the large-sample limit, $t^* = \argmax_{t \in \{1, \dots, \tau\}} \phi(\bfX \mid \hcalG^{(t)}, \calD)$ selects $\alpha^{(t^*)} = \hat{a}_n$ that results in the true CPDAG $\hcalG_n^{(t^*)} (\alpha_n) = {\calG^*}$, with high probability. 

Given that the arguments hold for any $\alpha_n$ as long as $\alpha_n \geq a_n$ when $n$ is large, it is not necessary for $\alpha_n$ to converge to zero. In particular, \eqref{eq:thm_path} holds for any fixed $\alpha_n = \alpha \in (0, 1)$. 

\end{proof}

Note that while \autoref{thm:path} generously allows for any $\alpha_n \in (0, 1)$, in practice a small threshold is desirable in the interest of efficiency. Since $\Pr(p_{n;i,j|\bfk} \leq \alpha_n) = \alpha_n$, the expected length of the solution path $\tau$ may be approximately bounded between $\abs{\bfE^*}$ and $\abs{\bfE^*} + \frac{p(p-1)}{2} \alpha_n$. The bounds may be further regulated if approximate knowledge of the sparsity of ${\calG^*}$ is known. 

Furthermore, an inspection of the proof of \autoref{thm:path} more generally indicates that the solution path returned by \autoref{alg:path} with the same $\tau$ and $\alpha^{(\tau)}$ contains the CPDAG of the underlying DAG ${\calG^*}$ if there exists some $a \leq \alpha$ such that
\begin{align*}
	\max_{i, j \in \bfE^*} \Phi_{ij} 
		\leq a <  \min_{i, j \not\in \bfE^*} \Phi_{ij}
\end{align*} 
holds when \eqref{eq:exist2} is investigated with threshold $\alpha$, which can hold even for a finite sample size $n$.

\section{Constraint-based Edge Orientation}\label{sec:orient}

In this section, we review established constraint-based strategies for determining edge orientations, summarized in \autoref{alg:cpdag}.

We begin by discussing the work of \cite{verma1991equivalence} in determining edge orientations. Under faithfulness, knowledge about the conditional independence relationships between variables can be used to detect the existence of v-structures (defined in \autoref{sec:equivalence}) as follows. 
A triplet of nodes $(i, k, j)$ configured $i \adjacent k \adjacent j$ with $i$ and $j$ not adjacent, called an \emph{unshielded triple}, is a v-structure if and only if 
\begin{align}\label{eq:vs1}
	\exists \bfk \subseteq \bfV \setminus \{i, j, k\}  \text{ such that } \indep{X_i}{X_j}{\bfX_{\bfk}}.
\end{align}
It is easy to see that any other directed configuration of $i \adjacent k \adjacent j$ requires $k$ to separate $i$ and $j$ \citep[Lemma 5.1.3]{spirtes2000}. To investigate this criterion, the separation sets $\bfS$ are recorded throughout the estimation process (e.g., line \ref{algl:sep} of \autoref{alg:pc}), defined as $\bfS(i, j) = \bfS(j, i) \subseteq \bfV \setminus \{i, j \}$ such that $\indep{X_i}{X_j}{\bfX_{\bfS(i, j)}}$ for distinct nodes $i$ and $j$. Assuming accurate separation sets, every v-structure in $\calG$ can be recovered according to \eqref{eq:vs1}, guaranteeing extension of the skeleton of $\calG$ to its pattern, to which Meek's rules can be straightforwardly applied to obtain its CPDAG (see \autoref{sec:equivalence}). 

Alternatively, considering the Markov condition, an unshielded triple $(i, k, j)$ in $\calG$ with adjacencies $\bfN_i^\calG$ and $\bfN_j^\calG$ can be correctly identified as a v-structure by investigating the criteria \eqref{eq:vs1} limited to sets $\bfk$ such that $\bfX_\bfk \subseteq \bfN_i^\calG$ or sets $\bfk$ such that $\bfX_\bfk \subseteq \bfN_j^\calG$. \cite{margaritis2003} proposed the following criteria for identifying $(i, k, j)$ as a v-structure: 
\begin{align}
	&\notindep{X_i}{X_j}{\bfX_{\bfk}} \text{ for all } \bfX_{\bfk} \subseteq \bfN_{i}^\calG \text{ or for all } \bfX_{\bfk} \subseteq \bfN_{j}^\calG \text{ with } k \in \bfk, \label{eq:vs2}
\end{align}
investigating the smaller of $\abs{\bfN_i^\calG}$ and $\abs{\bfN_j^\calG}$ for efficiency. For constraint-based algorithms that do not record $\bfS$, we empirically prefer investigating \eqref{eq:vs2} over \eqref{eq:vs1} for both general well-performance and efficiency. As with \eqref{eq:vs1}, the correctness of \eqref{eq:vs2} guarantees detection of all and only the v-structures in a DAG $\calG$ faithful to $P$ given its skeleton and conditional independence oracles.

\begin{algorithm}
\caption{\texttt{skel-to-cpdag($\calG$, $\bfS$, $\indepinfo$)}}
\label{alg:cpdag}
\begin{algorithmic}[1]
\Require{undirected graph $\calG$, and separation sets $\bfS$ or conditional independence information $\indepinfo$}
\Ensure{CPDAG $\calG$}
\State{initialize $\bfU = \emptyset$ to store v-structures} \label{algl:cpdag_start_detect}
\ForAll{triplets of nodes configured $i \adjacent k \adjacent j$ in $\calG$ where $i$ and $j$ are not adjacent}
\If{$\bfS$ is supplied}
\If{$k \not\in \bfS(i, j)$}\label{algl:detect_v}
\State{store $(i, k, j)$ in $\bfU$}
\EndIf
\Else
\State{let $q = \argmin_{l \in \{i, j \}} \abs{\bfN_l^{\calG}}$}
\If {$\notindep{X_i}{X_j}{\bfX_{\bfk}} \text{ for all } \bfX_{\bfk} \subseteq \bfN_{q}^\calG \text{ with } k \in \bfk$}
\State{store $(i, k, j)$ in $\bfU$}
\EndIf
\EndIf
\EndFor \label{algl:cpdag_end_detect}
\ForAll{triplets $(i, k, j) \in \bfU$} \label{algl:cpdag_start_orient}
\State{orient $i \to k$ and $j \to k$ in $\calG$}
\EndFor \label{algl:cpdag_end_orient}
\State{repeatedly apply Meek's rules R1-4 to $\calG$ until no rule is applicable} \label{algl:meek}
\end{algorithmic}
\end{algorithm}

\section{Details Regarding Statistical Evaluations}\label{sec:evaluations}

In this section, we provide basic details omitted from the body of the paper for the sake of brevity. 

We develop here the notation and evaluation of the popular $G^2$ log-likelihood ratio test of independence for empirical estimation of conditional independence in $P$ \citep{spirtes2000}. For further details and examples, we refer to \cite{neapolitan2004} 10.3.1. Let $n[x_i]$ denote the number of counts for which random variable $X_i = x_i$ in $n$ data samples $\calD$, with the definition extending similarly to $n[x_i, x_j]$ and so on. For $\bfk \subseteq \bfV$, let $n[x_i, x_j, \bfx_{\bfk}]$ denote the number of counts for which $X_i = x_i$, $X_j = x_j$, and $\bfX_{\bfk}$ attains one of its $\prod_{k \in \bfk} r_k$ state configurations $\bfx_{\bfk}$ in $\calD$. Then the $G^2$ test statistic for testing $H_0: \indep{X_i}{X_j}{X_k}$ is calculated as
\begin{align}\label{eq:gsquare}
	G^2_{ij | \bfk} 
		&= 2 \sum_{x_i, x_j, \bfx_{\bfk}} n[x_i, x_j, \bfx_{\bfk}] \log \left( \frac{n[x_i, x_j, \bfx_{\bfk}] \cdot n[\bfx_{\bfk}]}{n[x_i, x_j] \cdot n[x_j, \bfx_{\bfk}]} \right).
\end{align}
The equation can be applied to test marginal independence with $\bfk = \emptyset$. Under $H_0$, the $G^2$ statistic %for testing $\indep{X_i}{X_j}{\bfX_{\bfk}}$ 
is asymptotically $\chi^2_f$ distributed with $f = (r_i - 1) (r_j - 1) \prod_{k \in \bfk} r_k$ degrees of freedom. Then for a chosen significance level $\alpha$, 
\begin{align*} 
	\rmPr(\chi^2_f > G^2_{ij|\bfk}) \leq \alpha ~\Rightarrow~ \text{reject $H_0$: }\notindep{X_i}{X_j}{\bfX_{\bfk}}.
\end{align*}
Note that the PC algorithm (\autoref{alg:pc}) operates in a somewhat backwards fashion where the initialized complete graph on $\bfV$ assumes all distinct pairs of variables dependent rather than independent. A node pair $i$ and $j$ is then disconnected if
\begin{align*} 
	\rmPr(\chi^2_f > G^2_{ij|\bfk}) > \alpha ~\Rightarrow~ \text{accept $H_0$: }\indep{X_i}{X_j}{\bfX_{\bfk}}
\end{align*}
for some considered conditioning set $\bfX_\bfk$.

The mutual information $\I(X_i, X_j)$ serves as a similarity measure between discrete random variables $X_i$ and $X_j$. It may be interpreted as the Kullback-Leibler divergence between the joint probability distribution and the product of the marginals, and it is empirically calculated for $n$ data observations as 
\begin{align*} 
	\hat{\I}(X_i, X_j) 
		= \sum_{x_i, x_j}
		\frac{n[x_i, x_j]}{n} \log \left( \frac{n[x_i, x_j] / n}{n[x_i]/n \cdot n[x_j] / n} \right)
\end{align*}
where $n[x_i,x_j]$ is the counts of the instances of the $n$ observations that satisfy $X_i = x_i$ and $X_j = x_j$, with corresponding definitions for $n[x_i]$ and $n[x_j]$. It is easy to see that $G^2_{ij} = 2n \cdot \hat{\I} (X_i, X_j)$.

Lastly, we include the penalized multinomial log-likelihood score with penalty parameter $\lambda$. Following the notation from \eqref{eq:gsquare}, let $\pa{i}$ represent one of the $q_i$ unique state configurations of $\Pa{i}$. Given the Bayesian network DAG structure $\calG$ and empirical data $\mathcal{D}$ from $P$, the penalized score is computed as
\begin{align}\label{eq:penll}
	\begin{split}
	\phi (\calG, \mathcal{D})
		&= \sum_{i, x_i, \pa{i}} n[x_{i}, \pa{i}] \log \left( \frac{n[x_{i}, \pa{i}]}{n[\pa{i}]} \right) - \lambda \sum_{i} (r_i - 1) q_i \\
		&= \sum_{i=1}^p \left( \sum_{x_i, \pa{i}} n[x_{i}, \pa{i}] \log \left( \frac{n[x_{i}, \pa{i}]}{n[\pa{i}]} \right) - \lambda (r_i - 1) q_i \right) \\
		&= \sum_{i=1}^p \phi(X_i, \Pa{i}).
	\end{split}
\end{align}
In the penalty term, $\sum_{i=1}^p (r_i - 1) q_i$ encodes model complexity, encouraging sparsity in the structure of $\calG$. In the Bayesian information criterion (BIC) score, $\lambda = \frac{1}{2} \log(n)$ \citep{schwarz1978}.

\pagebreak

\section{Supplementary Information}\label{sec:supplement}

This section contains supplementary figures and tables with additional and more detailed results.

\subsection{Additional Results}\label{sec:additional}

\begin{figure}[H]
  \includegraphics[width=164mm]{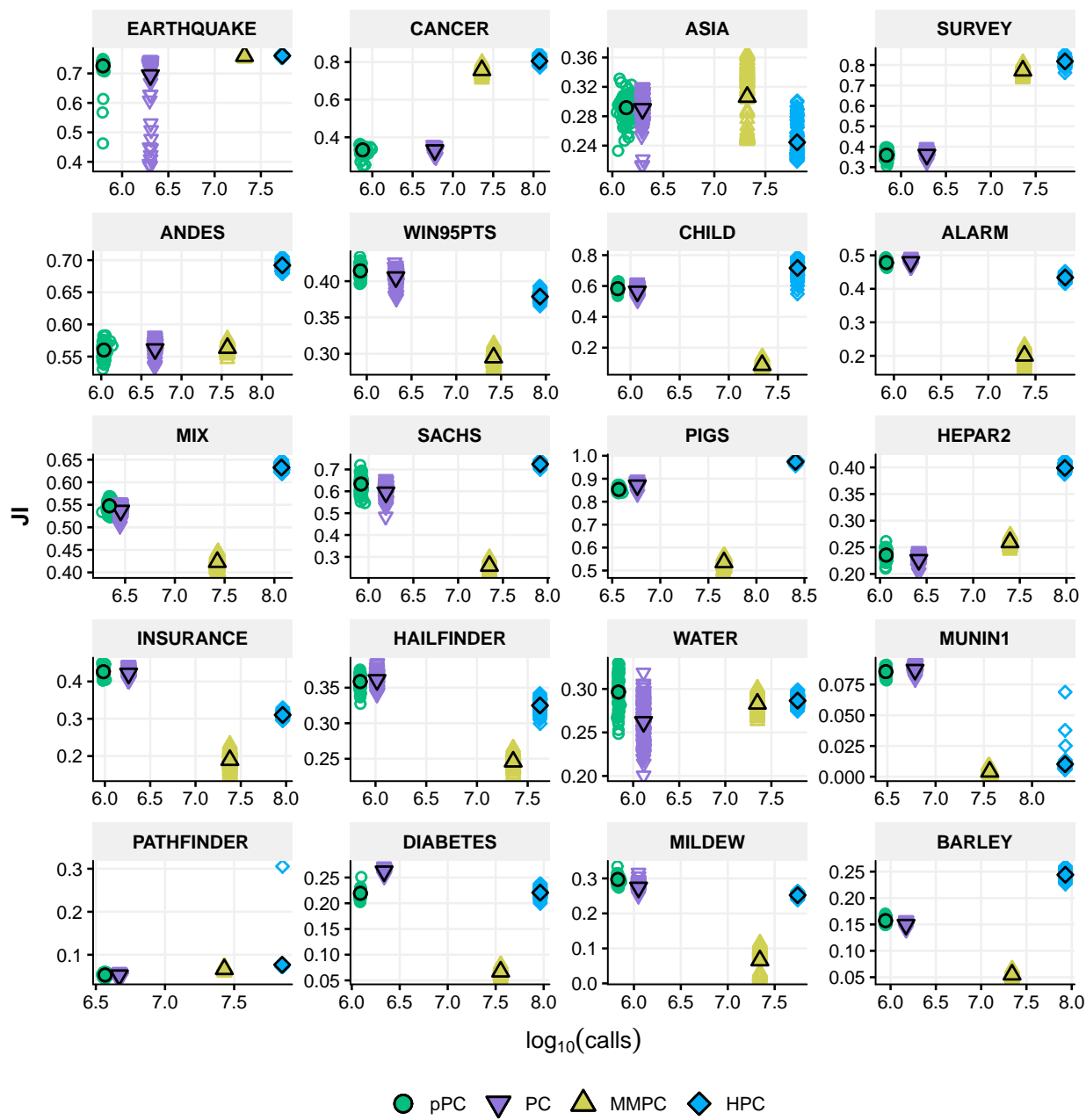}
  \caption{Visual accuracy and efficiency comparisons amongst pPC, PC, MMPC, and HPC. Filled shapes with black outlines indicate averages, whereas colored outlines represent results for individual datasets. The highest JI and $\log_{10}$ total calls were reported for PC, MMPC, and HPC for ten executions with $\alpha \in \mathcal{A}$, whereas the results for pPC represent a single execution with PATH and parameters $\alpha = 0.1$ and $\tau = 10$.}
  \label{fig:constraint_points}
\end{figure}

\begin{table}[H]
\begin{centering}
\scalebox{0.9}{
\centering
\begin{tabular}{lrrrrrrrr}
\toprule
  & pPC & PC & MMPC & HPC & pPC & PC & MMPC & HPC\\
\midrule
\addlinespace
%  & EARTHQUAKE &  &  &  & CANCER &  &  & \\
\multicolumn{1}{l}{ } & \multicolumn{4}{l}{EARTHQUAKE} & \multicolumn{4}{l}{CANCER} \\
\cmidrule(l{3pt}r{3pt}){2-5} \cmidrule(l{3pt}r{3pt}){6-9}
P & 850.3 & 869.0 & 874.0 & 858.2 & 1058.6 & 1035.8 & 1073.0 & 1120.4\\
TP & 820.7 & 801.4 & \textbf{853.1} & 847.0 & 543.5 & 537.3 & 945.8 & \textbf{1000.4}\\
R & 26.9 & 49.4 & 16.6 & \textbf{8.8} & 497.2 & 497.6 & 112.9 & \textbf{78.2}\\
FP & 2.7 & 18.2 & 4.3 & \textbf{2.3} & 17.9 & \textbf{0.8} & 14.3 & 41.8\\
SHD & 285.0 & 319.8 & \textbf{254.3} & 258.3 & 597.3 & 586.5 & 191.5 & \textbf{164.4}\\
JI & 0.726 & 0.694 & 0.759 & \textbf{0.760} & 0.333 & 0.331 & 0.757 & \textbf{0.805}\\
Calls & \textbf{5.793} & 6.306 & 7.323 & 7.731 & \textbf{5.880} & 6.775 & 7.360 & 8.074\\
\addlinespace
% 1 & ASIA &  &  &  & SURVEY &  &  & \\
\multicolumn{1}{l}{ } & \multicolumn{4}{l}{ASIA} & \multicolumn{4}{l}{SURVEY} \\
\cmidrule(l{3pt}r{3pt}){2-5} \cmidrule(l{3pt}r{3pt}){6-9}
P & 767.3 & 767.5 & 906.8 & 712.9 & 1284.2 & 1278.7 & 1307.2 & 1346.2\\
TP & 453.4 & 451.1 & \textbf{503.3} & 384.0 & 694.2 & 696.2 & 1156.9 & \textbf{1211.8}\\
R & 309.6 & 311.3 & 342.2 & \textbf{299.1} & 576.8 & 579.2 & 146.9 & \textbf{117.7}\\
FP & \textbf{4.3} & 5.1 & 61.2 & 29.8 & 13.2 & \textbf{3.3} & 3.4 & 16.7\\
SHD & \textbf{793.8} & 797.1 & 800.8 & 888.7 & 666.0 & 654.0 & 193.5 & \textbf{151.9}\\
JI & 0.292 & 0.290 & \textbf{0.306} & 0.245 & 0.359 & 0.361 & 0.773 & \textbf{0.818}\\
Calls & \textbf{6.137} & 6.296 & 7.323 & 7.812 & \textbf{5.838} & 6.287 & 7.357 & 7.827\\
\addlinespace
% 2 & ANDES &  &  &  & WIN95PTS &  &  & \\
\multicolumn{1}{l}{ } & \multicolumn{4}{l}{ANDES} & \multicolumn{4}{l}{WIN95PTS} \\
\cmidrule(l{3pt}r{3pt}){2-5} \cmidrule(l{3pt}r{3pt}){6-9}
P & 1792.2 & 1768.9 & 1557.5 & 1773.2 & 1354.7 & 1302.8 & 1025.6 & 1174.5\\
TP & 1449.3 & 1443.1 & 1370.2 & \textbf{1643.3} & \textbf{1036.4} & 1005.7 & 731.0 & 922.6\\
R & 313.0 & 318.1 & 149.5 & \textbf{118.1} & 259.2 & 264.3 & 273.7 & \textbf{236.1}\\
FP & 29.9 & \textbf{7.7} & 37.7 & 11.8 & 59.2 & 32.8 & 20.9 & \textbf{15.9}\\
SHD & 825.6 & 809.6 & 912.5 & \textbf{613.5} & \textbf{1206.8} & 1211.2 & 1474.0 & 1277.3\\
JI & 0.560 & 0.561 & 0.563 & \textbf{0.692} & \textbf{0.414} & 0.405 & 0.295 & 0.379\\
Calls & \textbf{6.033} & 6.670 & 7.572 & 8.255 & \textbf{5.929} & 6.325 & 7.417 & 7.935\\
\addlinespace
% 3 & CHILD &  &  &  & ALARM &  &  & \\
\multicolumn{1}{l}{ } & \multicolumn{4}{l}{CHILD} & \multicolumn{4}{l}{ALARM} \\
\cmidrule(l{3pt}r{3pt}){2-5} \cmidrule(l{3pt}r{3pt}){6-9}
P & 1299.3 & 1303.6 & 977.2 & 1310.9 & 1268.8 & 1286.3 & 1011.8 & 1403.3\\
TP & 959.9 & 942.3 & 183.2 & \textbf{1092.4} & 956.9 & \textbf{964.9} & 450.4 & 935.5\\
R & 338.4 & 359.1 & 777.1 & \textbf{207.5} & \textbf{309.7} & 321.1 & 549.1 & 451.2\\
FP & \textbf{1.0} & 2.1 & 16.9 & 11.1 & 2.2 & \textbf{0.2} & 12.3 & 16.7\\
SHD & 350.1 & 368.9 & 1142.8 & \textbf{227.6} & 734.3 & \textbf{724.4} & 1250.8 & 770.2\\
JI & 0.583 & 0.564 & 0.087 & \textbf{0.717} & 0.478 & \textbf{0.480} & 0.200 & 0.434\\
Calls & \textbf{5.869} & 6.070 & 7.340 & 7.700 & \textbf{5.921} & 6.179 & 7.383 & 7.815\\
\addlinespace
% 4 & MIX &  &  &  & SACHS &  &  & \\
\multicolumn{1}{l}{ } & \multicolumn{4}{l}{MIX} & \multicolumn{4}{l}{SACHS} \\
\cmidrule(l{3pt}r{3pt}){2-5} \cmidrule(l{3pt}r{3pt}){6-9}
P & 1478.9 & 1482.2 & 1167.3 & 1567.7 & 1791.3 & 1798.6 & 1372.0 & 1809.4\\
TP & 1187.3 & 1173.5 & 904.9 & \textbf{1334.8} & 1400.1 & 1349.5 & 656.0 & \textbf{1525.3}\\
R & 285.5 & 305.2 & 244.8 & \textbf{223.7} & 391.2 & 449.1 & 714.6 & \textbf{283.7}\\
FP & 6.0 & \textbf{3.5} & 17.6 & 9.2 & \textbf{0.0} & \textbf{0.0} & 1.3 & 0.4\\
SHD & 695.6 & 707.0 & 989.7 & \textbf{551.4} & 421.9 & 472.5 & 1167.3 & \textbf{297.1}\\
JI & 0.548 & 0.537 & 0.423 & \textbf{0.633} & 0.633 & 0.595 & 0.259 & \textbf{0.724}\\
Calls & \textbf{6.342} & 6.455 & 7.427 & 8.072 & \textbf{5.916} & 6.194 & 7.349 & 7.916\\
\bottomrule
\end{tabular}
}
\par\end{centering}
\begin{centering}
\protect\caption{Detailed results for the first ten networks comparing pPC with PATH against established constraint-based algorithms PC, MMPC, and HPC. Calls reports $\log_{10}$ calls. The highest JI and $\log_{10}$ total calls were reported for PC, MMPC, and HPC for ten executions with $\alpha \in \mathcal{A}$, whereas the results for pPC represent a single execution with PATH and parameters $\alpha = 0.1$ and $\tau = 10$. Best values are provided in boldface.}
\label{tab:detailed_constraint1}
\par\end{centering}
\end{table}

\begin{table}[H]
\begin{centering}
\scalebox{0.9}{
\centering
\begin{tabular}{lrrrrrrrr}
\toprule
  & pPC & PC & MMPC & HPC & pPC & PC & MMPC & HPC\\
\midrule
% 5 & PIGS &  &  &  & HEPAR2 &  &  & \\
\multicolumn{1}{l}{ } & \multicolumn{4}{l}{PIGS} & \multicolumn{4}{l}{HEPAR2} \\
\cmidrule(l{3pt}r{3pt}){2-5} \cmidrule(l{3pt}r{3pt}){6-9}
P & 2171.2 & 2172.3 & 1533.1 & 2192.7 & 1286.3 & 1280.7 & 1015.3 & 1687.5\\
TP & 2004.2 & 2027.0 & 1295.9 & \textbf{2158.1} & \textbf{640.6} & 619.6 & 636.9 & 1073.6\\
R & 166.9 & 145.3 & 213.6 & \textbf{14.6} & 645.6 & 661.1 & \textbf{310.1} & 589.8\\
FP & 0.1 & \textbf{0.0} & 23.7 & 19.9 & \textbf{0.0} & \textbf{0.0} & 68.3 & 24.0\\
SHD & 177.9 & 155.0 & 909.8 & \textbf{43.9} & 1436.4 & 1457.4 & 1508.4 & \textbf{1027.3}\\
JI & 0.853 & 0.871 & 0.536 & \textbf{0.974} & 0.235 & 0.226 & 0.259 & \textbf{0.399}\\
Calls & \textbf{6.568} & 6.761 & 7.663 & 8.407 & \textbf{6.066} & 6.416 & 7.395 & 7.984\\
\addlinespace
% 6 & INSURANCE &  &  &  & HAILFINDER &  &  & \\
\multicolumn{1}{l}{ } & \multicolumn{4}{l}{INSURANCE} & \multicolumn{4}{l}{HAILFINDER} \\
\cmidrule(l{3pt}r{3pt}){2-5} \cmidrule(l{3pt}r{3pt}){6-9}
P & 1434.5 & 1437.8 & 991.2 & 1609.4 & 927.3 & 910.1 & 830.5 & 1113.5\\
TP & \textbf{1062.0} & 1054.3 & 495.0 & 883.2 & \textbf{651.4} & 649.7 & 468.1 & 651.1\\
R & \textbf{371.4} & 383.5 & 473.4 & 696.0 & \textbf{219.9} & 249.0 & 316.2 & 461.9\\
FP & 1.1 & \textbf{0.0} & 22.8 & 30.1 & 56.0 & 11.3 & 46.1 & \textbf{0.4}\\
SHD & \textbf{1059.1} & 1065.7 & 1647.8 & 1266.9 & 945.5 & 902.6 & 1119.0 & \textbf{890.4}\\
JI & \textbf{0.426} & 0.421 & 0.189 & 0.310 & 0.359 & \textbf{0.361} & 0.246 & 0.325\\
Calls & \textbf{5.983} & 6.262 & 7.379 & 7.962 & \textbf{5.849} & 6.015 & 7.357 & 7.621\\
\addlinespace
% 7 & WATER &  &  &  & MUNIN1 &  &  & \\
\multicolumn{1}{l}{ } & \multicolumn{4}{l}{WATER} & \multicolumn{4}{l}{MUNIN1} \\
\cmidrule(l{3pt}r{3pt}){2-5} \cmidrule(l{3pt}r{3pt}){6-9}
P & 1328.6 & 1299.5 & 1313.6 & 1481.5 & 528.2 & 548.3 & 267.2 & 713.6\\
TP & 830.2 & 747.8 & 797.7 & \textbf{842.8} & 181.4 & \textbf{185.9} & 8.3 & 25.1\\
R & \textbf{495.0} & 551.3 & 513.4 & 616.3 & 307.3 & 325.6 & \textbf{216.0} & 636.2\\
FP & 3.4 & \textbf{0.4} & 2.5 & 22.4 & 39.5 & \textbf{36.8} & 42.9 & 52.3\\
SHD & \textbf{1476.1} & 1555.6 & 1507.8 & 1482.6 & \textbf{1634.0} & 1626.8 & 1810.6 & 1803.2\\
JI & \textbf{0.296} & 0.262 & 0.283 & 0.287 & 0.085 & \textbf{0.087} & 0.004 & 0.010\\
Calls & \textbf{5.840} & 6.113 & 7.349 & 7.785 & \textbf{6.484} & 6.787 & 7.562 & 8.343\\
\addlinespace
% 8 & PATHFINDER &  &  &  & DIABETES &  &  & \\
\multicolumn{1}{l}{ } & \multicolumn{4}{l}{PATHFINDER} & \multicolumn{4}{l}{DIABETES} \\
\cmidrule(l{3pt}r{3pt}){2-5} \cmidrule(l{3pt}r{3pt}){6-9}
P & 303.2 & 296.1 & 330.6 & 1185.3 & 1029.7 & 1059.7 & 700.4 & 1495.9\\
TP & 114.4 & 112.3 & 143.5 & \textbf{223.7} & 551.2 & \textbf{642.3} & 171.9 & 638.0\\
R & 182.2 & 179.1 & \textbf{170.3} & 916.9 & 458.5 & \textbf{409.0} & 465.0 & 673.6\\
FP & 6.6 & \textbf{4.8} & 16.9 & 44.7 & 20.1 & \textbf{8.4} & 63.5 & 184.2\\
SHD & 1860.2 & 1860.5 & 1841.4 & \textbf{1789.0} & 1504.0 & \textbf{1401.1} & 1926.6 & 1581.2\\
JI & 0.053 & 0.052 & 0.067 & \textbf{0.077} & 0.219 & \textbf{0.262} & 0.067 & 0.221\\
Calls & \textbf{6.567} & 6.669 & 7.429 & 7.847 & \textbf{6.090} & 6.337 & 7.552 & 7.964\\
\addlinespace
% 9 & MILDEW &  &  &  & BARLEY &  &  & \\
\multicolumn{1}{l}{ } & \multicolumn{4}{l}{MILDEW} & \multicolumn{4}{l}{BARLEY} \\
\cmidrule(l{3pt}r{3pt}){2-5} \cmidrule(l{3pt}r{3pt}){6-9}
P & 1047.5 & 1029.3 & 783.3 & 1493.0 & 733.5 & 713.5 & 585.2 & 1322.5\\
TP & 678.2 & 632.5 & 163.7 & \textbf{686.0} & 385.0 & 365.6 & 140.0 & \textbf{671.6}\\
R & \textbf{362.8} & 396.8 & 613.5 & 806.0 & \textbf{328.9} & 329.7 & 421.7 & 580.4\\
FP & 6.6 & \textbf{0.0} & 6.2 & 1.0 & 19.6 & \textbf{18.3} & 23.5 & 70.6\\
SHD & 1241.5 & 1280.5 & 1755.5 & \textbf{1227.9} & 1735.6 & 1753.7 & 1984.6 & \textbf{1500.0}\\
JI & \textbf{0.297} & 0.274 & 0.066 & 0.252 & 0.157 & 0.149 & 0.055 & \textbf{0.244}\\
Calls & \textbf{5.828} & 6.049 & 7.343 & 7.745 & \textbf{5.944} & 6.171 & 7.341 & 7.930\\
\bottomrule
\end{tabular}
}
\par\end{centering}
\begin{centering}
\protect\caption{Detailed results for the last ten networks comparing pPC with PATH against established constraint-based algorithms PC, MMPC, and HPC. Calls reports $\log_{10}$ calls. The highest JI and $\log_{10}$ total calls were reported for PC, MMPC, and HPC for ten executions with $\alpha \in \mathcal{A}$, whereas the results for pPC represent a single execution with PATH and parameters $\alpha = 0.1$ and $\tau = 10$. Best values are provided in boldface.}
\label{tab:detailed_constraint2}
\par\end{centering}
\end{table}

\begin{figure}[H]
  \includegraphics[width=164mm]{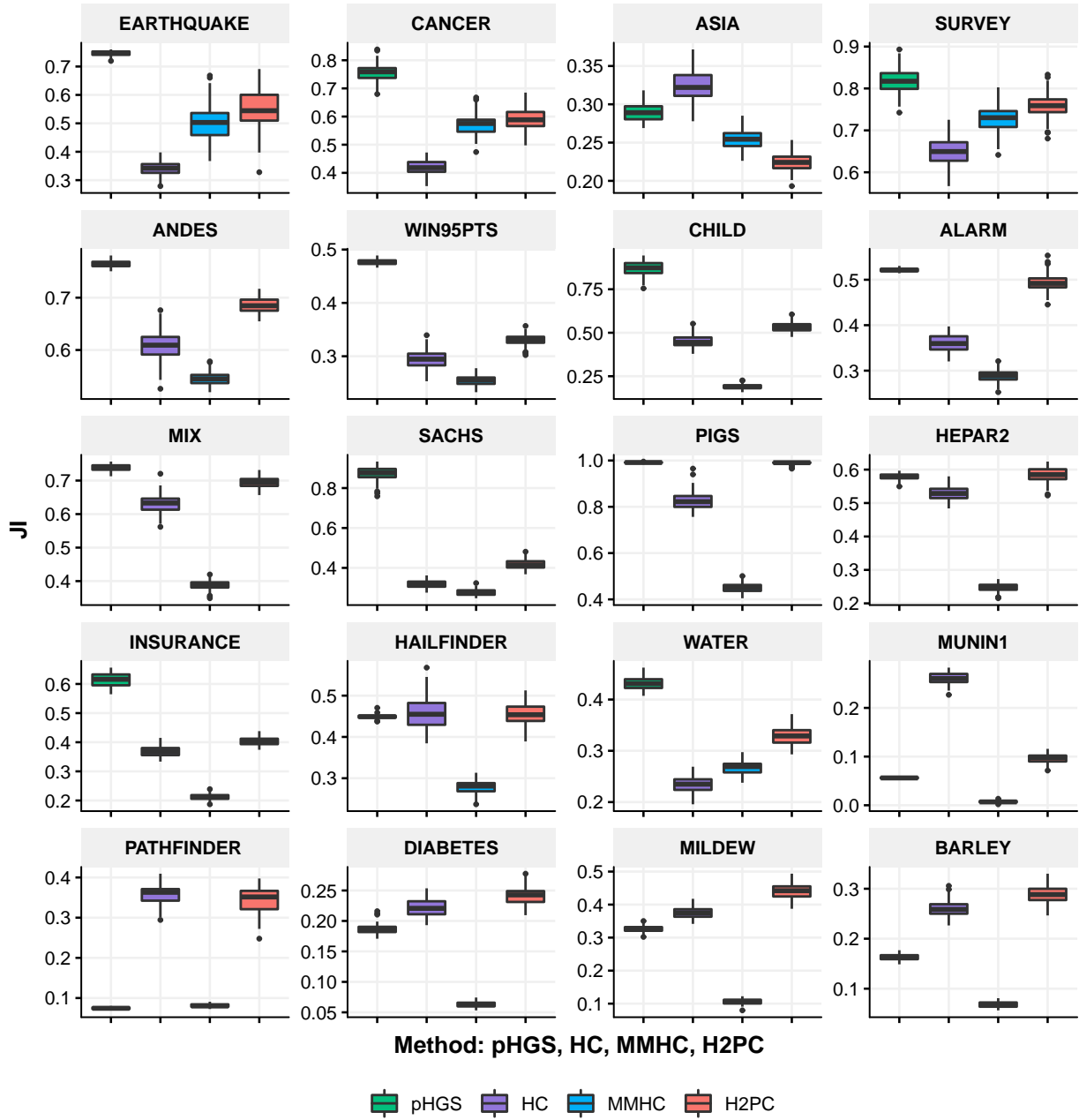}
  \caption{Visual representation of select accuracy results comparing pHGS against perfectly restricted hill-climbing from an empty graph (GSC$^*$) and established structure learning algorithms HC, MMHC, and H2PC. The highest JI were reported for MMHC and H2PC for ten executions with $\alpha \in \mathcal{A}$, and pHGS was executed with parameters $\alpha = 0.05$ and $\tau = 10$.}
  \label{fig:established_boxplots}
\end{figure}

\begin{table}[H]
\begin{centering}
\scalebox{0.9}{
\centering
\begin{tabular}{lrrrrrrrr}
\toprule
  & pHGS & HC & MMHC & H2PC & pHGS & HC & MMHC & H2PC\\
\midrule
\addlinespace
%  & EARTHQUAKE &  &  &  & CANCER &  &  & \\
\multicolumn{1}{l}{ } & \multicolumn{4}{l}{EARTHQUAKE} & \multicolumn{4}{l}{CANCER} \\
\cmidrule(l{3pt}r{3pt}){2-5} \cmidrule(l{3pt}r{3pt}){6-9}
P & 914.2 & 1210.5 & 874.7 & 865.3 & 1138.9 & 1306.5 & 1107.0 & 1094.6\\
TP & \textbf{862.2} & 588.1 & 657.2 & 693.8 & \textbf{974.5} & 716.7 & 808.9 & 822.2\\
R & \textbf{23.6} & 313.4 & 208.2 & 161.1 & \textbf{103.6} & 364.3 & 255.9 & 252.3\\
FP & 28.4 & 309.1 & \textbf{9.3} & 10.4 & 60.7 & 225.5 & 42.2 & \textbf{20.0}\\
SHD & \textbf{269.2} & 824.0 & 455.1 & 419.6 & \textbf{209.2} & 631.8 & 356.3 & 320.8\\
JI & \textbf{0.746} & 0.341 & 0.500 & 0.548 & \textbf{0.758} & 0.419 & 0.570 & 0.590\\
Calls & \textbf{5.756} & 6.378 & 6.321 & 6.732 & \textbf{5.795} & 6.444 & 6.326 & 7.074\\
\addlinespace
% 1 & ASIA &  &  &  & SURVEY &  &  & \\
\multicolumn{1}{l}{ } & \multicolumn{4}{l}{ASIA} & \multicolumn{4}{l}{SURVEY} \\
\cmidrule(l{3pt}r{3pt}){2-5} \cmidrule(l{3pt}r{3pt}){6-9}
P & 903.2 & 1432.0 & 852.3 & 702.0 & 1337.5 & 1432.2 & 1327.6 & 1341.7\\
TP & 481.6 & \textbf{653.3} & 424.4 & 356.2 & \textbf{1206.5} & 1093.4 & 1124.7 & 1160.7\\
R & 373.3 & 446.1 & 407.7 & \textbf{326.0} & \textbf{120.9} & 241.2 & 190.1 & 171.6\\
FP & 48.3 & 332.5 & 20.2 & \textbf{19.8} & 10.1 & 97.5 & 12.8 & \textbf{9.4}\\
SHD & \textbf{809.6} & 922.2 & 838.9 & 906.6 & \textbf{150.6} & 351.1 & 235.1 & 195.8\\
JI & 0.289 & \textbf{0.323} & 0.254 & 0.224 & \textbf{0.817} & 0.649 & 0.726 & 0.760\\
Calls & \textbf{5.871} & 6.510 & 6.323 & 6.812 & \textbf{5.786} & 6.506 & 6.327 & 6.823\\
\addlinespace
% 2 & ANDES &  &  &  & WIN95PTS &  &  & \\
\multicolumn{1}{l}{ } & \multicolumn{4}{l}{ANDES} & \multicolumn{4}{l}{WIN95PTS} \\
\cmidrule(l{3pt}r{3pt}){2-5} \cmidrule(l{3pt}r{3pt}){6-9}
P & 1901.3 & 2572.1 & 1549.8 & 1770.7 & 1348.4 & 2334.8 & 1024.0 & 1167.4\\
TP & 1796.7 & \textbf{1823.7} & 1337.3 & 1633.3 & \textbf{1139.7} & 1028.4 & 649.9 & 832.0\\
R & \textbf{52.8} & 270.2 & 185.2 & 130.5 & \textbf{161.9} & 501.5 & 358.2 & 323.8\\
FP & 51.8 & 478.3 & 27.3 & \textbf{6.9} & 46.7 & 804.8 & 15.9 & \textbf{11.6}\\
SHD & \textbf{500.1} & 899.6 & 935.0 & 618.6 & \textbf{1091.0} & 1960.4 & 1550.0 & 1363.6\\
JI & \textbf{0.765} & 0.610 & 0.544 & 0.686 & \textbf{0.476} & 0.295 & 0.254 & 0.330\\
Calls & \textbf{5.980} & 6.776 & 6.442 & 7.201 & \textbf{5.901} & 6.627 & 6.402 & 6.926\\
\addlinespace
% 3 & CHILD &  &  &  & ALARM &  &  & \\
\multicolumn{1}{l}{ } & \multicolumn{4}{l}{CHILD} & \multicolumn{4}{l}{ALARM} \\
\cmidrule(l{3pt}r{3pt}){2-5} \cmidrule(l{3pt}r{3pt}){6-9}
P & 1299.4 & 1438.7 & 969.6 & 1298.0 & 1299.0 & 1828.8 & 1020.6 & 1399.9\\
TP & \textbf{1213.3} & 851.6 & 364.8 & 905.3 & \textbf{1024.2} & 932.5 & 606.5 & 1020.0\\
R & \textbf{85.9} & 449.2 & 597.8 & 391.5 & \textbf{270.9} & 511.8 & 403.2 & 371.8\\
FP & \textbf{0.2} & 137.9 & 7.0 & 1.1 & \textbf{3.9} & 384.6 & 10.9 & 8.1\\
SHD & \textbf{95.8} & 595.4 & 951.2 & 404.8 & \textbf{668.7} & 1141.1 & 1093.5 & 677.1\\
JI & \textbf{0.871} & 0.450 & 0.191 & 0.532 & \textbf{0.522} & 0.361 & 0.288 & 0.493\\
Calls & \textbf{5.834} & 6.486 & 6.338 & 6.698 & \textbf{5.879} & 6.577 & 6.382 & 6.812\\
\addlinespace
% 4 & MIX &  &  &  & SACHS &  &  & \\
\multicolumn{1}{l}{ } & \multicolumn{4}{l}{MIX} & \multicolumn{4}{l}{SACHS} \\
\cmidrule(l{3pt}r{3pt}){2-5} \cmidrule(l{3pt}r{3pt}){6-9}
P & 1572.7 & 1937.4 & 1171.8 & 1563.5 & 1801.5 & 1944.2 & 1373.0 & 1779.1\\
TP & 1465.8 & \textbf{1473.2} & 851.8 & 1410.0 & \textbf{1688.0} & 911.6 & 690.9 & 1061.1\\
R & \textbf{88.7} & 231.7 & 303.2 & 150.2 & \textbf{113.6} & 855.0 & 680.8 & 718.0\\
FP & 18.1 & 232.5 & 16.8 & \textbf{3.3} & \textbf{0.0} & 177.6 & 1.3 & \textbf{0.0}\\
SHD & \textbf{429.3} & 636.3 & 1042.0 & 470.3 & \textbf{134.0} & 1088.0 & 1132.4 & 760.9\\
JI & \textbf{0.739} & 0.630 & 0.388 & 0.695 & \textbf{0.873} & 0.320 & 0.276 & 0.418\\
Calls & \textbf{6.281} & 6.631 & 6.397 & 6.950 & \textbf{5.887} & 6.577 & 6.346 & 6.905\\
\bottomrule
\end{tabular}
}
\par\end{centering}
\begin{centering}
\protect\caption{Detailed results for the first ten networks comparing pHGS against established structure learning algorithms HC, MMHC, and H2PC. Calls reports $\log_{10}$(Calls). The highest JI were reported for MMHC and H2PC for ten executions with $\alpha \in \mathcal{A}$, and pHGS was executed with parameters $\alpha = 0.05$ and $\tau = 10$. Best values are pvodied in boldface.}
\label{tab:detailed_established1}
\par\end{centering}
\end{table}

\begin{table}[H]
\begin{centering}
\scalebox{0.9}{
\centering
\begin{tabular}{lrrrrrrrr}
\toprule
  & pHGS & HC & MMHC & H2PC & pHGS & HC & MMHC & H2PC\\
\midrule
\addlinespace
% 5 & PIGS &  &  &  & HEPAR2 &  &  & \\
\multicolumn{1}{l}{ } & \multicolumn{4}{l}{PIGS} & \multicolumn{4}{l}{HEPAR2} \\
\cmidrule(l{3pt}r{3pt}){2-5} \cmidrule(l{3pt}r{3pt}){6-9}
P & 2171.9 & 2298.6 & 1533.8 & 2172.9 & 1456.1 & 1781.5 & 988.3 & 1576.6\\
TP & \textbf{2167.5} & 2026.6 & 1152.9 & 2166.2 & 1295.2 & 1335.3 & 606.1 & \textbf{1347.5}\\
R & \textbf{4.4} & 152.5 & 357.2 & 6.7 & \textbf{145.1} & 296.4 & 336.4 & 219.4\\
FP & \textbf{0.0} & 119.5 & 23.6 & \textbf{0.0} & 15.8 & 149.8 & 45.7 & \textbf{9.6}\\
SHD & \textbf{14.5} & 275.0 & 1052.7 & 15.8 & 797.5 & 891.6 & 1516.5 & \textbf{739.0}\\
JI & \textbf{0.991} & 0.827 & 0.450 & 0.990 & 0.579 & 0.529 & 0.247 & \textbf{0.585}\\
Calls & \textbf{6.545} & 6.856 & 6.650 & 7.228 & \textbf{6.028} & 6.517 & 6.377 & 6.912\\
\addlinespace
% 6 & INSURANCE &  &  &  & HAILFINDER &  &  & \\
\multicolumn{1}{l}{ } & \multicolumn{4}{l}{INSURANCE} & \multicolumn{4}{l}{HAILFINDER} \\
\cmidrule(l{3pt}r{3pt}){2-5} \cmidrule(l{3pt}r{3pt}){6-9}
P & 1526.2 & 2003.5 & 995.0 & 1536.9 & 994.9 & 1574.8 & 838.9 & 1101.1\\
TP & \textbf{1388.4} & 1108.5 & 545.4 & 1051.4 & 785.7 & \textbf{974.6} & 518.7 & 825.3\\
R & \textbf{134.6} & 582.6 & 426.8 & 469.7 & \textbf{114.5} & 328.6 & 273.1 & 275.5\\
FP & \textbf{3.3} & 312.4 & 22.8 & 15.8 & 94.6 & 271.6 & 47.1 & \textbf{0.3}\\
SHD & \textbf{734.9} & 1323.9 & 1597.5 & 1084.4 & 849.9 & 838.0 & 1069.4 & \textbf{715.9}\\
JI & \textbf{0.615} & 0.368 & 0.212 & 0.404 & 0.449 & \textbf{0.456} & 0.279 & 0.455\\
Calls & \textbf{5.952} & 6.584 & 6.376 & 6.958 & \textbf{5.822} & 6.524 & 6.354 & 6.617\\
\addlinespace
% 7 & WATER &  &  &  & MUNIN1 &  &  & \\
\multicolumn{1}{l}{ } & \multicolumn{4}{l}{WATER} & \multicolumn{4}{l}{MUNIN1} \\
\cmidrule(l{3pt}r{3pt}){2-5} \cmidrule(l{3pt}r{3pt}){6-9}
P & 1356.6 & 1842.9 & 1312.7 & 1449.6 & 524.2 & 1692.7 & 267.8 & 710.7\\
TP & \textbf{1103.0} & 786.0 & 763.9 & 926.1 & 122.5 & \textbf{717.4} & 14.3 & 217.8\\
R & \textbf{252.7} & 632.6 & 547.4 & 517.4 & 365.0 & 318.1 & \textbf{210.8} & 444.4\\
FP & \textbf{1.0} & 424.3 & 1.4 & 6.1 & \textbf{36.8} & 657.2 & 42.7 & 48.5\\
SHD & \textbf{1201.0} & 1941.4 & 1540.4 & 1383.0 & 1690.3 & 1715.8 & 1804.4 & \textbf{1606.6}\\
JI & \textbf{0.432} & 0.234 & 0.268 & 0.328 & 0.056 & \textbf{0.261} & 0.007 & 0.096\\
Calls & \textbf{5.812} & 6.553 & 6.344 & 6.779 & 6.465 & \textbf{6.449} & 6.562 & 7.343\\
\addlinespace
% 8 & PATHFINDER &  &  &  & DIABETES &  &  & \\
\multicolumn{1}{l}{ } & \multicolumn{4}{l}{PATHFINDER} & \multicolumn{4}{l}{DIABETES} \\
\cmidrule(l{3pt}r{3pt}){2-5} \cmidrule(l{3pt}r{3pt}){6-9}
P & 298.1 & 1430.4 & 330.9 & 1107.7 & 1033.6 & 2209.5 & 702.5 & 1422.5\\
TP & 157.4 & 894.3 & 172.4 & \textbf{787.1} & 482.1 & \textbf{770.0} & 161.8 & 671.2\\
R & \textbf{135.4} & 348.0 & 142.8 & 274.2 & 530.7 & 741.0 & \textbf{477.7} & 602.1\\
FP & \textbf{5.4} & 188.2 & 15.7 & 46.4 & \textbf{20.7} & 698.5 & 63.1 & 149.1\\
SHD & 1816.0 & 1261.8 & 1811.3 & \textbf{1227.2} & 1573.6 & 1963.5 & 1936.3 & \textbf{1512.8}\\
JI & 0.075 & \textbf{0.358} & 0.081 & 0.345 & 0.186 & 0.222 & 0.063 & \textbf{0.241}\\
Calls & 6.552 & 6.437 & \textbf{6.429} & 6.844 & \textbf{6.069} & 6.650 & 6.552 & 6.964\\
\addlinespace
% 9 & MILDEW &  &  &  & BARLEY &  &  & \\
\multicolumn{1}{l}{ } & \multicolumn{4}{l}{MILDEW} & \multicolumn{4}{l}{BARLEY} \\
\cmidrule(l{3pt}r{3pt}){2-5} \cmidrule(l{3pt}r{3pt}){6-9}
P & 1100.7 & 1599.2 & 778.1 & 1314.7 & 742.8 & 1603.1 & 587.9 & 1170.5\\
TP & 741.8 & 958.0 & 257.7 & \textbf{986.4} & 398.8 & \textbf{763.1} & 171.7 & 732.5\\
R & 358.8 & 399.3 & 515.3 & \textbf{327.9} & \textbf{324.5} & 468.2 & 393.5 & 400.8\\
FP & \textbf{0.1} & 241.9 & 5.1 & 0.4 & \textbf{19.6} & 371.9 & 22.7 & 37.2\\
SHD & 1171.3 & 1196.9 & 1660.4 & \textbf{927.0} & 1721.8 & 1709.8 & 1952.0 & \textbf{1405.7}\\
JI & 0.327 & 0.375 & 0.106 & \textbf{0.440} & 0.163 & 0.260 & 0.068 & \textbf{0.289}\\
Calls & \textbf{5.801} & 6.529 & 6.343 & 6.744 & \textbf{5.909} & 6.507 & 6.342 & 6.929\\
\bottomrule
\end{tabular}
}
\par\end{centering}
\begin{centering}
\protect\caption{Detailed results for the last ten networks comparing pHGS against established structure learning algorithms HC, MMHC, and H2PC. Calls reports $\log_{10}$(Calls). The highest JI and average calls were reported for MMHC and H2PC for ten executions with $\alpha \in \mathcal{A}$. The highest JI were reported for MMHC and H2PC for ten executions with $\alpha \in \mathcal{A}$, and pHGS was executed with parameters $\alpha = 0.05$ and $\tau = 10$. Best values are pvodied in boldface.}
\label{tab:detailed_established2}
\par\end{centering}
\end{table}

\subsection{Gaussian Results}\label{sec:gaussian}

Here, we report some preliminary results on the application of HGI to ARGES, as well as briefly discuss the design and asymptotic behavior of these algorithms.

Adaptively restricted greedy equivalence search (ARGES) is a hybrid adaptation of greedy equivalence search (GES) that achieves asymptotic recovery of the CPDAG of the underlying DAG given either a consistent estimation algorithm for the skeleton or the conditional independence graph (CIG), the two variants respectively referred to as ARGES-skeleton and ARGES-CIG \citep{nandy2018}. The CIG of the joint distribution of variables $\bfX = \{X_1, \dots, X_p\}$ represented by nodes $\bfV = \{1, \dots, p \}$ is the undirected graph where $i \adjacent j$ if and only if $\notindep{X_i}{X_j}{\bfX \setminus \{X_i, X_j\}}$, and is a supergraph of the skeleton.

For our simulations, we again considered the network structures in 
\autoref{tab:networks}
% \mtref{main-tab:networks} 
% Main Table 1 
for our simulations. For each DAG structure, the edge weights $\beta_{ij}$ were sampled uniformly from $[-1, -0.1] \cup [0.1, 1]$ for $(i, j) \in \bfE$, with $\beta_{ij} = 0$ otherwise. Similarly, the zero mean error variances $\sigma^2_i$, $i = 1, \dots, p$ were drawn uniformly from $[1, 2]$. For each network configuration, we generated $N = 10$ datasets with sample sizes $n \in \{100, 250, 500, 1000, 2500, 5000\}$. 
For each dataset and algorithm, we applied the algorithm with ten score penalties $\lambda \in \mathcal{L} \coloneqq \{0.5, 0.7, \dots, 2.3\}$ (BIC penalty being $\lambda = 0.5$).  Note that due to the restriction to the true skeleton, $\lambda$ functions largely to adjust the search traversal through the space of equivalence classes rather than control the sparsity. The detection of v-structures for HGI was accomplished with a single significance level threshold $\alpha = 0.01$. 

The results are summarized in \autoref{fig:gaussian_hgi}. For ARGES without HGI, we took the highest JI of ARGES-skeleton and ARGES-CIG. For ARGES with HGI, we report the JI obtained by executing ARGES-skeleton initialized with the CPDAG of the initial DAG obtained by HGI. Both ARGES and ARGES with HGI achieve greater structural accuracy as $n$ increases, though ARGES with HGI is generally more accurate and approaches $\text{JI} = 1$ more quickly. The addition of HGI consistently results in typically 14-18\% improvement in JI, typically slightly increasing with sample size. ARGES with HGI only obtained worse JI results than ARGES for 2.7\% of the executions, with fewer than 0.4\% executions more than 4\% worse.

\begin{figure}[H]
  \includegraphics[width=\linewidth]{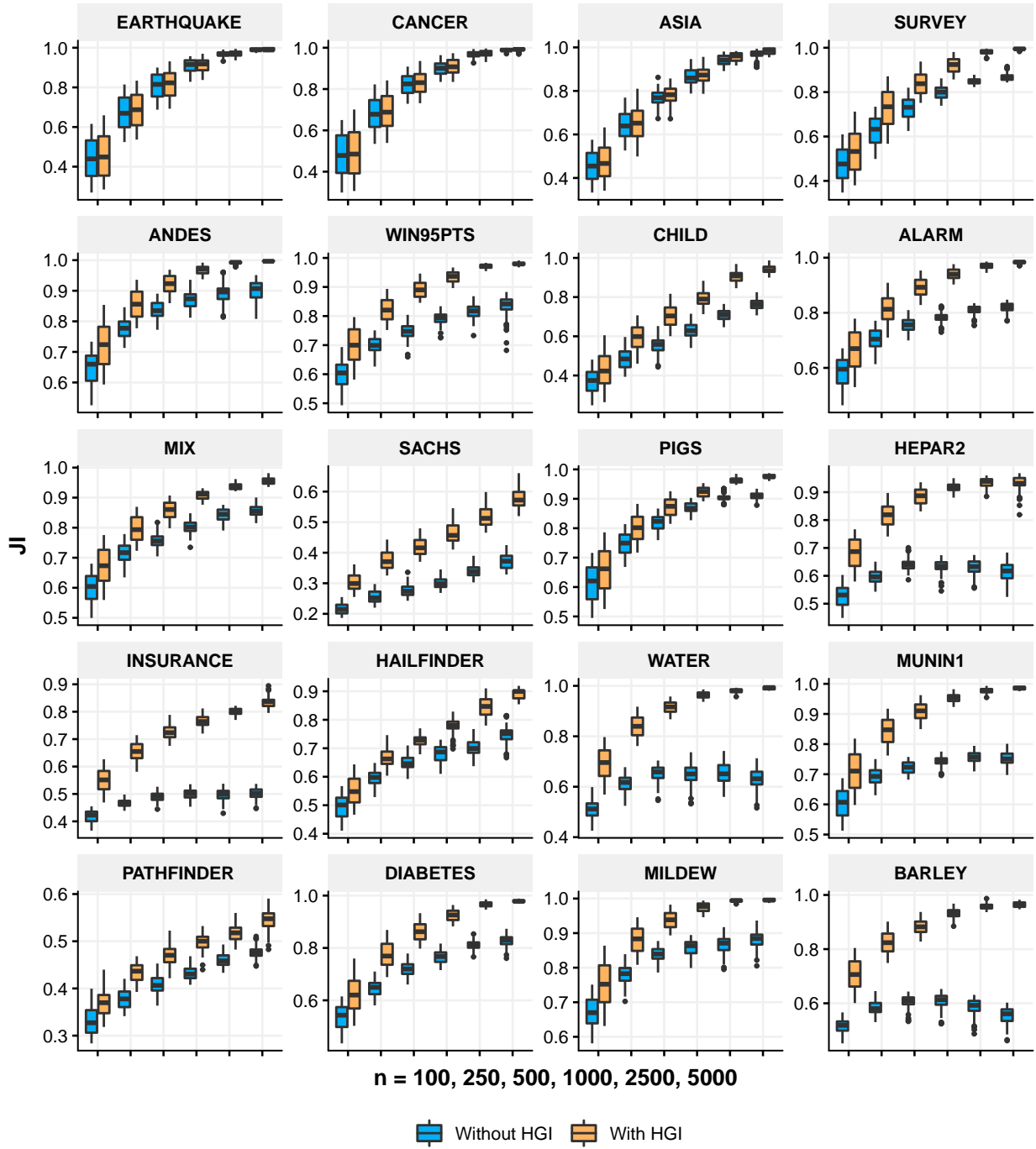}
  \caption{Boxplots comparing the JI of ARGES without and with HGI for various networks and sample sizes, restricted to the skeleton or CIG of the underlying DAG.}
  \label{fig:gaussian_hgi}
\end{figure}

\end{document}